\documentclass[lettersize,journal]{IEEEtran}

\usepackage{amsmath,amsthm,amssymb,amsfonts}
\usepackage{xcolor,colortbl}
\usepackage{graphicx,color}
\usepackage{multirow} 
\usepackage{textcomp}
\usepackage{dirtytalk}
\usepackage{hyperref}
\usepackage{cleveref}
\usepackage{xcolor}
\usepackage{cite}
\usepackage{color}
\usepackage{booktabs}
\usepackage{algorithmic}
\usepackage{algorithm}
\usepackage{array}
\usepackage[caption=false,font=normalsize,labelfont=sf,textfont=sf]{subfig}
\usepackage{textcomp}
\usepackage{stfloats}
\usepackage{url}
\usepackage{verbatim}

\definecolor{darkred}{rgb}{1, 0.1, 0.3}
\definecolor{darkgreen}{rgb}{0.5, 0.8, 0.1}
\definecolor{darkpurple}{rgb}{1.0, 0, 1.0}
\definecolor{darkblue}{rgb}{0, 0, 1.0}
\definecolor{Gray}{gray}{0.85}
\newcolumntype{a}{>{\columncolor{Gray}}c}
\newcolumntype{b}{>{\columncolor{white}}c}

\newtheorem{proposition}{Proposition}
\newtheorem{definition}{Definition}
\newtheorem{corollary}{Corollary}
\newtheorem{theorem}{Theorem}
\newtheorem{problem}{Problem}
\newtheorem{example}{Example}
\newtheorem{lemma}{Lemma}
\newtheorem{claim}{Claim}

\newtheorem{remark}{Remark}

\providecommand{\bydef}{\overset{\text{def}}{=}}
\DeclareMathOperator*{\argmin}{arg\,min}

\begin{document}

\title{Principal Component Analysis in Space Forms}

\author{Puoya Tabaghi, Michael Khanzadeh, Yusu Wang, and Siavash Mirarab
        % <-this % stops a space
\thanks{P. Tabaghi and Y. Wang are with the Hal{\i}c{\i}o\u{g}lu Data Science Institute, UCSD (e-mail: \href{mailto:ptabaghi@ucsd.edu}{ptabaghi@ucsd.edu}; \href{mailto:yusuwang@ucsd.edu}{yusuwang@ucsd.edu}). M. Khanzadeh was formerly with the Computer Science and Engineering Department, UCSD. S. Mirarab is with the Electrical and Computer Engineering Department, UCSD (e-mail: \href{mailto:smirarabbaygi@ucsd.edu}{smirarabbaygi@ucsd.edu}).}}

\markboth{Journal of \LaTeX\ Class Files,~Vol.~14, No.~8, August~2021}%
{Shell \MakeLowercase{\textit{et al.}}: A Sample Article Using IEEEtran.cls for IEEE Journals}


\maketitle

\begin{abstract}
Principal Component Analysis (PCA) is a workhorse of modern data science. While PCA assumes the data conforms to Euclidean geometry, for specific data types, such as hierarchical and cyclic data structures, other spaces are more appropriate. We study PCA in space forms; that is, those with constant curvatures. At a point on a Riemannian manifold, we can define a Riemannian affine subspace based on a set of tangent vectors. Finding the optimal low-dimensional affine subspace for given points in a space form amounts to dimensionality reduction. Our Space Form PCA (SFPCA) seeks the affine subspace that best represents a set of manifold-valued points with the minimum projection cost. We propose proper cost functions that enjoy two properties: (1) their optimal affine subspace is the solution to an eigenequation, and (2) optimal affine subspaces of different dimensions form a nested set. These properties provide advances over existing methods, which are mostly iterative algorithms with slow convergence and weaker theoretical guarantees. We evaluate the proposed SFPCA on real and simulated data in spherical and hyperbolic spaces. We show that it outperforms alternative methods in estimating true subspaces (in simulated data) with respect to convergence speed or accuracy, often both.
\end{abstract}

\begin{IEEEkeywords}
principal component analysis, Riemannian manifolds, hyperbolic and spherical spaces.
\end{IEEEkeywords}

\section{Introduction}
\IEEEPARstart{G}{iven} a set of multivariate points, principal component analysis (PCA) finds orthogonal basis vectors so that different components of the data, in the new coordinates, become uncorrelated and the leading bases carry the largest projected variance of the points. PCA is related to factor analysis~\cite{thurstone1931multiple}, Karhunen-Lo\'eve expansion, and singular value decomposition~\cite{wold1987principal} --- with a history going back to the 18th century~\cite{stewart1993early}. The modern formalism of PCA goes back to the work of Hotelling~\cite{hotelling1933analysis}. Owing to its interpretability and flexibility, PCA has been an indispensable tool in data science applications~\cite{jolliffe2002principal}. The PCA formulation has been studied numerously in the literature. Tipping and Bishop~\cite{tipping1999probabilistic} established a connection between factor analysis and PCA in a probabilistic framework. Other extensions have been proposed~\cite{vidal2005generalized}, e.g., Gaussian processes~\cite{lawrence2003gaussian} and sensible~\cite{roweis1997algorithms}, Bayesian~\cite{bishop1998bayesian}, sparse~\cite{jolliffe2003modified,zou2006sparse,cai2013sparse,guan2009sparse}, and Robust PCA~\cite{xu2010robust}.

PCA's main features are its linearity and nested optimality of subspaces with different dimensions. PCA uses a linear transformation to extract features. Thus, applying PCA to non-Euclidean data ignores their geometry, produces points that may not belong to the original space, and breaks downstream applications relying on this geometry~\cite{fletcher2004principal,jiang2022learning,tabaghi2020hyperbolic,fletcher2004principal2}.

We focus on \emph{space forms}: complete, simply connected Riemannian manifolds of dimension $d \geq 2$ and constant curvature --- spherical, Euclidean, or hyperbolic spaces (positive, zero, and negative curvatures)~\cite{lee2006riemannian}. Space forms have gained attention in the machine learning community due to their ability to represent many forms of data. Hyperbolic spaces are suitable for hierarchical structures~\cite{sonthalia2020tree,tabaghi2020geometry,chien2021highly,tabaghi2020hyperbolic,chien2022hyperaid}, biological data~\cite{klimovskaia2020poincare,zhou2021hyperbolic}, and phylogenetic trees~\cite{jiang2022learning}. Spherical spaces find application in text embeddings~\cite{meng2019spherical}, longitudinal data~\cite{dai2018principal}, and cycle-structures in graphs~\cite{gu2018learning}.

To address the shortcomings of Euclidean PCA for non-Euclidean data, several authors propose Riemannian PCA methods~\cite{rahman2005multiscale,tournier2009motion,anirudh2015elastic,fletcher2003statistics,fletcher2004principal,sommer2014optimization,huckemann2010intrinsic,lazar2017scale,pennec2018barycentric,sommer2010manifold}. 
Riemannian manifolds generally lack a vector space structure~\cite{tabaghi2021linear}, posing challenges for defining \emph{principal components}. A common approach for dimensionality reduction of manifold-valued data relies on tangent spaces. Fletcher \emph{et al.}~\cite{fletcher2003statistics} propose a cost to quantify the quality of a Riemannian affine subspace but use a heuristic approach, principal geodesics analysis (PGA), to optimize it: (1) the base point (intrinsic mean) is the solution to a fixed-point problem, (2) Euclidean PCA in tangent space estimates the low-rank tangent vectors. Even more principled approaches do not readily yield tractable solutions necessary for analyzing large-scale data~\cite{sommer2014optimization}, as seen in spherical and hyperbolic PCAs~\cite{liu2019spherical,dai2018principal,chami2021horopca}.

Despite recent progress, PCA in space forms remains inadequately explored. In general, cost-based Riemannian (e.g., spherical and hyperbolic) PCAs rely on finding the optimal Riemannian affine subspace by minimizing a nonconvex function. The cost function, proxy, or methodology is usually inspired by the $\ell_2^2$ cost, with no definite justification for it~\cite{chami2021horopca,dai2018principal,liu2019spherical,fletcher2004principal,fletcher2003statistics,huckemann2010intrinsic}. These algorithms rely on iterative methods to estimate the Riemannian affine subspaces, e.g., gradient descent, fixed-point iterations, and proximal alternate minimization, and they are slow to converge and require parameter tuning. There is also no guarantee that estimated Riemannian affine subspaces form a total order under inclusion (i.e., optimal higher-dimensional subspaces include lower-dimensional ones) unless they perform cost minimization in a greedy (suboptimal) fashion by building high-dimensional subspaces based on previously estimated low-dimensional subspaces. Notably, Chakraborty \emph{et al.} propose a greedy PGA for space forms by estimating one principal geodesic at a time~\cite{chakraborty2016efficient}. They derive an analytic formula for projecting a point onto a parameterized geodesic. This simplifies the projection step of the PGA. However, we still have to solve a nonconvex optimization problem (with no theoretical guarantees) to estimate the principal geodesic at each iteration.

We address PCA limitations in space forms by proposing a \emph{closed-form, theoretically optimal, and computationally efficient} method to derive all principal geodesics at once. We begin with a differential geometric view of Euclidean PCA (\Cref{sec:Principal_Component_Analysis_Revisited}), followed by a generic description of Riemannian PCA (\Cref{sec:Principal_Components_Analysis_in_Riemannian_Manifolds}). In this view, a \emph{proper} PCA cost function must (1) naturally define a centroid for manifold-valued points and (2) yield \emph{theoretically optimal} affine subspaces forming a total order under inclusion. We introduce proper costs for spherical (\Cref{sec:spherical_PCA}) and hyperbolic (\Cref{sec:hyperbolic_PCA}) PCA problems. Minimizing each cost function leads to an eigenequation, which can be effectively solved. For hyperbolic PCA, the optimal affine subspace solves an eigenequation in Lorentzian space which is equipped with an indefinite inner product. These results give us efficient algorithms to derive hyperbolic principal components. Proofs are in the Appendix.

\subsection{Preliminaries and Notations} \label{supp:prelim}
Let $(\mathcal{M}, g)$ be a Riemannian manifold. The tangent space $T_p \mathcal{M}$ is the collection of all tangent vectors at $p \in \mathcal{M}$. The Riemannian metric $g_p : T_p \mathcal{M} \times T_p \mathcal{M} \rightarrow \mathbb{R}$ is given by a positive-definite inner product and depends smoothly on $p$. We use $g_p$ to define notions such as subspace, norms, and angles, similar to inner product spaces. For any subspace $H \subseteq T_p \mathcal{M}$, we define its \emph{orthogonal complement} as follows:
\begin{equation}\label{eq:orthogonal_complement}
H^{\perp} = \{ h^{\prime} \in T_{p} \mathcal{M} : g_{p}(h , h^{\prime}) = 0, \ \forall h \in H \} \subseteq T_p \mathcal{M}.
\end{equation}
The norm of $v \in T_p \mathcal{M}$ is $\|v\| \bydef \sqrt{g_p(v,v)}$. We denote the length of a smooth curve $\gamma : [0,1] \rightarrow \mathcal{M}$ as $L[\gamma] = \int_{0}^{1} \| \gamma^{\prime}(t) \| \, dt$. A geodesic $\gamma_{p_1,p_2}$ is the shortest-length path between $p_1$ and $p_2 \in \mathcal{M}$, that is, $\gamma_{p_1,p_2} = \argmin_{\gamma} L[\gamma] : \gamma(0) = p_1, \gamma(1) = p_2$.  Interpreting the parameter $t$ as \emph{time}, if a geodesic $\gamma(t)$ starts at $\gamma(0) = p \in \mathcal{M}$ with initial velocity $\gamma^{\prime}(0) = v \in T_p \mathcal{M}$, the exponential map $\mathrm{exp}_p(v)$ gives its position at $t = 1$. For $p$ and $x \in \mathcal{M}$, the logarithmic map $\mathrm{log}_p(x)$ gives the initial velocity to move (with constant speed) along the geodesic from $p$ to $x$ in one time step. A Riemannian manifold $\mathcal{M}$ is geodesically complete if the exponential and logarithmic maps, at every point $p \in \mathcal{M}$, are well-defined operators~\cite{gallier2020differential}. A submanifold $\mathcal{M}^{\prime}$ of a Riemannian manifold $(\mathcal{M},g)$ is \emph{geodesic} if any geodesic on $\mathcal{M}^{\prime}$ with its induced metric $g$ is also a geodesic on $(\mathcal{M},g)$. For $N \in \mathbb{N}$, we let $[N] \bydef \{1, \ldots, N\}$ and $[N]_0 \bydef [N] \cup \{0\}$. The variable $x_1$ is an element of the vector $x = (x_0, \ldots, x_{D-1})^{\top} \in \mathbb{R}^{D}$. It can also be an indexed vector, e.g., $x_1, \ldots, x_N \in \mathbb{R}^{D}$. This distinction will be clarified in the context. We use $\mathbb{E}_{N}[\cdot]$ to denote the empirical mean of its inputs with indices in $[N]$.

\section{Principal Component Analysis --- Revisited} \label{sec:Principal_Component_Analysis_Revisited}
Similar to the notion by Pearson~\cite{pearson1901liii}, PCA finds the optimal low-dimensional affine space to represent data. Let $p \in \mathbb{R}^D$ and let the column span of $H \in \mathbb{R}^{D \times K}$ be a subspace. For the affine subspace $p+H$, PCA assumes the following cost:
\begin{equation*}
    \mathrm{cost}(p+H|\mathcal{X}) =  \mathbb{E}_{N}\Big[ f \circ d\big( x_n , \mathcal{P}_{p+H}( x_n ) \big) \Big],
\end{equation*}
where $\mathcal{X} = \{ x_n \in \mathbb{R}^D: n \in [N] \}$, $\mathcal{P}_{p+H}( x_n ) = \argmin_{x \in  p + H } d(x, x_n)$, $d(\cdot, \cdot)$ computes the $\ell_2$ distance, and the distortion function $f(x) = x^2$. This formalism relies on $(1)$ an affine subspace $p + H$, $(2)$ the projection operator $\mathcal{P}_{p+H}$, and $(3)$ the distortion function $f$. To generalize affine subspaces to Riemannian manifolds, consider parametric lines:
\begin{equation}\label{eq:point_line_and_normal_line}
    \gamma_{p,x}(t) = (1-t)p+tx, \ \mbox{ and } \ \gamma_{h^{\prime}}(t) = p + t h^{\prime},
\end{equation}
where $h^{\prime} \in H^{\perp}$, the orthogonal complement of the subspace $H$; see \Cref{fig:affine_subspaces} $(a,b)$. We reformulate affine subspaces as:
\begin{align*}
    p+H &= \{ x \in \mathbb{R}^{D}: \langle x-p , h^{\prime} \rangle = 0, \text{ for all } h^{\prime} \in H^{\perp} \}\\
    &= \{ x \in \mathbb{R}^{D}: \langle \gamma_{p,x}'(0) , \gamma_{h^{\prime}}'(0) \rangle = 0, \text{ for all } h^{\prime} \in H^{\perp} \},
\end{align*}
where $\langle \cdot, \cdot \rangle$ is the dot product and $\gamma^{\prime}(t_0) \bydef \frac{d}{dt}\gamma(t)|_{t=t_0}$.

\begin{definition}\label{def:affine_subspace_tangent_vectors}
An affine subspace is a set of points, e.g., $x$, where there exists $p \in \mathbb{R}^{D}$ such that (tangent of) the line $\gamma_{x,p}$ (i.e., $\gamma^{\prime}_{p,x}(0)$) is normal to a set of tangent vectors at $p$.
\end{definition}
\Cref{def:affine_subspace_tangent_vectors} requires $\mathrm{dim}(H^{\perp})$ parameters to describe $p+H$. For example, in $\mathbb{R}^3$, we need two orthonormal vectors to represent a one-dimensional affine subspace; see \Cref{fig:affine_subspaces} $(a)$. We use \Cref{def:affine_subspace_tangent_vectors} since it describes affine subspaces in terms of \emph{lines} and \emph{tangent vectors}, not a global coordinate system.
\begin{figure*}[t!]
\centering
\includegraphics[width=.9\textwidth]{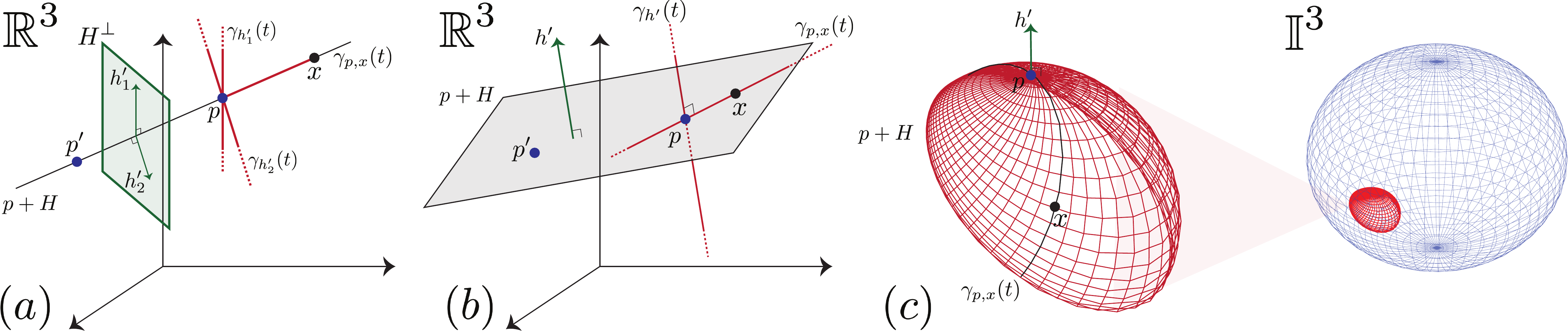}
\caption{ $(a, b)$ 
One- $(a)$ and two-dimensional $(b)$ affine subspaces in $\mathbb{R}^3$. We show subspaces ($H^{\perp}$) at point $p$ instead of the origin. We may define the same Riemannian affine subspace using other base points, e.g., $p^{\prime}$. $(c)$ Two-dimensional affine subspace in a hyperbolic space (Poincar\'e) where $h^{\prime} \in T_p\mathbb{I}^3 = \mathbb{R}^3$.} \label{fig:affine_subspaces}
\end{figure*}

\section{Riemannian Principal Component Analysis} \label{sec:Principal_Components_Analysis_in_Riemannian_Manifolds}
We next introduce Riemannian affine subspaces and then propose a generic framework for Riemannian PCA. 

\subsection{Riemannian Affine Subspaces}
The notion of affine subspaces can be extended to Riemannian manifolds using tangent subspaces of the Riemannian manifold $\mathcal{M}$~\cite{fletcher2003statistics}. The Riemannian affine subspace is the image of a subspace of $T_p \mathcal{M}$ under the exponential map, i.e.,
\begin{equation}
\label{eq:affine_subspace_manifold_exp}
    \mathcal{M}_H = \mathrm{exp}_p(H) \bydef  \{ \mathrm{exp}_p(h): h \in H \}, 
\end{equation}
where $H$ is a subspace of $T_p \mathcal{M}$ and $p \in \mathcal{M}$. Equation \ref{eq:affine_subspace_manifold_exp} is equivalent to the following definition. 
\begin{definition} \label{def:affine_subspace_manifold}
    Let $( \mathcal{M}, g)$ be a geodesically complete Riemannian manifold, $p \in \mathcal{M}$, and subspace $H \subseteq T_p \mathcal{M}$. We let $\mathcal{M}_{H}= \{ x \in \mathcal{M}: g_{p}( \mathrm{log}_p(x), h^{\prime} ) = 0, \ \forall h^{\prime} \in H^{\perp}\}$ where $H^{\perp}$ is the orthogonal complement of $H$; see equation \eqref{eq:orthogonal_complement}.
\end{definition}
The set $\mathcal{M}_{H}$ is a collection of points on geodesics originating at $p$ such that their initial velocities are normal to vectors in $H^{\perp}  \subseteq T_p \mathcal{M}$, or form a subspace $H \subseteq T_{p} \mathcal{M}$; cf. \Cref{def:affine_subspace_tangent_vectors}.
\begin{example}
    When $H$ is a one-dimensional subspace, then $\mathcal{M}_H$ contains geodesics that pass through $p$ with the initial velocity in $H$, i.e., $\mathcal{M}_H = \{ \gamma(t): \ \mbox{geodesic } \ \gamma \ \mbox{ where } \ \gamma (0) = p, \gamma^{\prime}(0) \in H, t \in \mathbb{R} \}$. Thus, with equation \eqref{eq:affine_subspace_manifold_exp}, geodesics are one-dimensional Riemannian affine subspaces. 
\end{example}
\begin{example} \label{ex:euclidean_from_riem}
    The Euclidean exponential map is $\mathrm{exp}_p(h) = p+h$; see \Cref{tab:diff_geom_ingredients}. Therefore, equation \eqref{eq:affine_subspace_manifold_exp} recovers the affine subspaces, i.e., $\mathrm{exp}_p(H) =\{ p+h: h \in H\} \bydef p+H$. 
\end{example}
Recall that, a nonempty set $V$ is a Euclidean affine subspace if and only if there exists $p \in V$ such that $\alpha (v_1-p)+p \in V$ and $(v_1-p) + (v_2-p) + p \in V$ for all $\alpha \in \mathbb{R}$ and $v_1, v_2 \in V$. We have a similar definition for Riemannian affine subspaces.
\begin{definition}
Let $(\mathcal{M},g)$ be a geodesically complete Riemannian manifold. Then, nonempty $V\subseteq \mathcal{M}$ is an affine set if and only if there exists $p \in V$ such that $\mathrm{exp}_p \big( \alpha \ \mathrm{log}_p(v_1) \big) \in V$ and $\mathrm{exp}_p \big( \mathrm{log}_p(v_1)+\mathrm{log}_p(v_2) \big) \in V$ for all $\alpha \in \mathbb{R}$ and $v_1, v_2 \in V$.
\end{definition}

\subsection{Proper Cost for Riemannian PCA}
Similar to Euclidean PCA, Riemannian PCA aims to find a (Riemannian) affine subspace with minimum average distortion between points and their projections. 

\begin{definition}\label{def:distortion}
Let $(\mathcal{M},g)$ be a geodesically complete Riemannian manifold that is equipped with distance function $d$, $p \in \mathcal{M}$, and subspace $H \subseteq T_p \mathcal{M}$. A geodesic projection of $x \in \mathcal{M}$ onto $\mathcal{M}_H$ is $\mathcal{P}_H(x) \in \argmin_{y \in \mathcal{M}_H} d(x,y)$. If $\argmin_{y \in \mathcal{M}_H} d(x,y) \neq \emptyset$, then $\min_{y \in \mathcal{M}_H} d(x,y) = \| \mathrm{log}_{\mathcal{P}_H(x)}(x) \|$ for any geodesic projection $\mathcal{P}_H(x)$.
\end{definition}
\begin{remark}
Projecting a manifold-valued point onto a Riemannian affine subspace is not a trivial task, often requiring the solution of a nonconvex optimization problem over a submanifold, i.e., solving $\argmin_{y \in \mathcal{M}_H} d(x,y)$. \Cref{def:distortion} states that if a solution exists (which is not always guaranteed), the projection distance must be equal to $\| \mathrm{log}_{\mathcal{P}_H(x)}(x) \|$. This also requires computing the logarithmic map, which may not be available for all Riemannian manifolds.
\end{remark}

In Euclidean PCA, minimizing the $\ell_2^2$ cost is equivalent to maximizing the variance of the projected data. To avoid confusing arguments regarding the notion of variance and centroid, we formalize the cost (parameterized by $\mathcal{M}_{H}$) in terms of the projection distance, viz.,
\begin{align}\label{eq:distortion_general}
\mathrm{cost}(\mathcal{M}_{H}| \mathcal{X})  = \mathbb{E}_N \big[ f(\|  \mathrm{log}_{\mathcal{P}_H(x_n)}(x_n) \|  )\big],    
\end{align}
where $\mathcal{X} = \{x_n \in \mathcal{M}: n \in [N] \}$ and $f: \mathbb{R}^{+} \rightarrow \mathbb{R}$ is a monotonically increasing distortion function. The projection point $\mathcal{P}_H(x)$ may not be unique. Its minimizer, if it exists, is the best affine subspace to represent manifold-valued points.
\begin{definition}
    Riemannian PCA aims to minimize the cost in equation \eqref{eq:distortion_general} --- for a specific choice of distortion function $f$.
\end{definition}

\emph{Choice of $f$.} The closed-form solution for the optimal Euclidean affine subspace is due to letting $f(x)=x^2$. This is a proper cost function with the following properties:
\begin{enumerate}
    \item \emph{Consistent Centroid.} The optimal $0$-dimensional affine subspace (a point) is the centroid of data points, i.e., $p^* = \argmin_{y \in \mathbb{R}^D}  \mathbb{E}_N [ f \circ d(x_n,y)] = \mathbb{E}_N [x_n]$. 
    \item \emph{Nested Optimality.} The optimal affine subspaces form a nested set, i.e., $p^{*} \subseteq \big(p + H_1\big)^{*} \subseteq  \cdots$ where $(p + H_d)^{*}$ is the optimal $d$-dimensional affine subspace.
\end{enumerate}
\begin{definition}
For Riemannian PCA, we call $\mathrm{cost}(\mathcal{M}_{H}| \mathcal{X})$ a proper cost function if its minimizers satisfy the consistent centroid and nested optimality conditions.
\end{definition}
Deriving the logarithm operator is not a trivial task for general Riemannian manifolds, e.g., the manifold of rank-deficient positive semidefinite matrices \cite{lahav2023procrustes}. Focusing on constant-curvature Riemannian manifolds, we propose distortion functions that, unlike existing methods, arrive at proper cost functions with closed-form optimal solutions.
\section{SPHERICAL PCA} \label{sec:spherical_PCA}
\begin{table*}[t]
  \centering
  \caption{Summary of relevant operators in Euclidean, spherical, and hyperbolic spaces.}\label{tab:diff_geom_ingredients}
  \begin{tabular}{ccccccc}
    \toprule
    $\mathcal{M}$ & $T_p \mathcal{M}$ & $g_p(u,v)$ & $\mathrm{log}_p(x): \theta = d(x,p)$ & $\mathrm{exp}_p(v)$ & $d(x,p)$ \\
    \midrule
    $\mathbb{R}^D$ & $\mathbb{R}^D$ & $\langle u, v \rangle$ & $x-p$ & $p+v$ & $\|x-p\|_2$\\
    $\mathbb{S}^D$ & $p^{\perp}$ & $\langle u, v \rangle$ & $\frac{\theta}{\sin(\theta)}(x - \cos(\theta) p)$ & $\cos(\sqrt{C}\|v\|)p + \sin(\sqrt{C}\|v\|) \frac{v}{\sqrt{C}\|v\|}$  &  $\frac{1}{\sqrt{C}}\arccos(C\langle x,p \rangle)$ \\
    $\mathbb{H}^D$ & $p^{\perp}$ & $[u,v]$ & $\frac{\theta}{\sinh(\theta)}(x - \cosh(\theta) p)$ & $\cosh(\sqrt{-C}\|v\|) p + \sinh(\sqrt{-C}\|v\|) \frac{v}{\sqrt{-C}\|v\|}$ & $\frac{1}{\sqrt{-C}}\operatorname{acosh}(C [x,p])$\\
  \bottomrule
\end{tabular}
\end{table*}
Consider the spherical manifold $(\mathbb{S}^{D}, g^{\mathbb{S}})$ with curvature $C>0$, where $\mathbb{S}^{D} = \{x \in \mathbb{R}^{D+1}: \langle x,x \rangle = C^{-1} \}$, a sphere with radius $C^{-\frac{1}{2}}$, and $g_p^{\mathbb{S}}(u,v) = \langle u, v \rangle$ computes the dot product of $u,v \in T_p \mathbb{S}^D = \{ x \in \mathbb{R}^{D+1}: \langle x,p \rangle = 0 \} \bydef p^{\perp}$.

\subsection{Spherical affine subspace and the projection operator} 
Let $p \in \mathbb{S}^D$ and the subspace $H \subseteq T_p \mathbb{S}^D = p^{\perp}$. Following \Cref{def:affine_subspace_manifold} and \Cref{tab:diff_geom_ingredients}, the \emph{spherical affine subspace} is:
\begin{equation*}%\label{eq:spherical_subspace}
    \mathbb{S}^D_{H} = \{ x \in \mathbb{S}^D: \langle x, h^{\prime} \rangle = 0, \ \forall h^{\prime} \in H^{\perp} \} = \mathbb{S}^D \cap (p \oplus H),
\end{equation*}
where $\oplus$ is the direct sum operator, i.e., $p \oplus H = \{ \alpha p + h: h \in H, \alpha \in \mathbb{R} \}$, and $H^{\perp}$ is the orthogonal complement of $H$; see equation \eqref{eq:orthogonal_complement}. This matches Pennec's notion of the
metric completion of exponential barycentric spherical subspace~\cite{pennec2018barycentric}.
\begin{claim}\label{claim:geodesic_submanifold_spherical}
    $\mathbb{S}^D_{H}$ is a geodesic submanifold. 
\end{claim}
There are orthogonal tangent vectors $h_1^{\prime}, \ldots, h_{K^{\prime}}^{\prime}$ that form a complete basis for $H^{\perp}$, i.e., $\langle h^{\prime}_{i}, h^{\prime}_{j} \rangle = C \delta_{i,j}$, where $\delta_{i,j} = 0$ if $i \neq j$  and $\delta_{i,j} = 1$ if $i=j$. Using these basis vectors, we derive a simple expression for the projection distance.

\begin{proposition}\label{prop:projection_and_distance_spherical}
For any $\mathbb{S}_H^D$ and $x \in \mathbb{S}^D$, we have
\[
    \min_{y \in \mathbb{S}_H^D}d(x,y) = C^{-\frac{1}{2}}\mathrm{acos}\Big(\sqrt{1 - \sum_{k \in [K^{\prime}]} \langle x,h^{\prime}_{k} \rangle^2 } \Big),
\]
where $\{h^{\prime}_{k^{\prime}}\}_{k^{\prime} \in [K^{\prime}]}$ are the complete orthogonal basis vectors of $H^{\perp}$. Both $\mathbb{S}_H^D$ and $\mathbb{S}^D$ have a fixed curvature $C > 0$.
\end{proposition}

For points at $C^{-\frac{1}{2}}\frac{\pi}{2}$ distance from $\mathbb{S}^D_H$, there is no unique projection onto the affine subspace. Nevertheless, \Cref{prop:projection_and_distance_spherical} provides a closed-form expression for the projection distance in terms of basis vectors for $H^{\perp}$. This distance monotonically increases with the length of the residual of $x$ onto $H^{\perp}$, i.e., $\sum_{k \in [K^{\prime}]} \langle x,h^{\prime}_k \rangle^2$. Since $\mathrm{dim}(H) \ll D$ is common, switching to the basis of $H$ helps us represent the projection distance with fewer parameters.
\begin{proposition}\label{prop:projection_and_distance_spherical_parallel}
For any $\mathbb{S}_H^D$ and $x \in \mathbb{S}^D$, we have
\[
    \min_{y \in \mathbb{S}_H^D}d(x,y) = C^{-\frac{1}{2}}\mathrm{acos} \Big( \sqrt{C^2\langle x,p \rangle^2 + \sum_{k \in [K]} \langle x,h_k \rangle^2 } \Big),
\]
where $\{ h_k\}_{k \in [K]}$ are complete orthogonal basis vectors of $H$.
\end{proposition}
Next, we derive an isometry between $\mathbb{S}^D_H$ and $\mathbb{S}^{\mathrm{dim}(H)}$ --- where both have the fixed curvature $C > 0$.
\begin{theorem} \label{thm:spherical_affine_subspace_isometry}
The isometry $\mathcal{Q}: \mathbb{S}_H^D \rightarrow \mathbb{S}^{K}$ and its inverse are 
\[
\mathcal{Q}(x) =  C^{-\frac{1}{2}}\begin{bmatrix}
    C\langle x,p \rangle \\
    \langle x,h_{1} \rangle \\
    \vdots \\
   \langle x,h_{K} \rangle
\end{bmatrix}, \ \mathcal{Q}^{-1} ( y ) =  C^{-\frac{1}{2}} (y_0 C p+ \sum_{k = 1}^{K} y_k h_{k}),
\]
where $\{ h_k\}_{k \in [K]}$ are complete orthogonal basis vectors of $H$.
\end{theorem}
\begin{corollary}\label{cor:a}
The dimension of $\mathbb{S}_H^D$ is $\mathrm{dim}(H)$.
\end{corollary}
Finally, we can provide an alternative view of spherical affine subspaces based on sliced unitary matrices $\{ G \in \mathbb{R}^{(D+1)\times (K+1)}: G^{\top}G = I_{K+1}\}$.
\begin{claim}\label{cl:spherical_alt}
For any $\mathbb{S}^D_H$, there is a sliced-unitary operator $G: \mathbb{S}^{\mathrm{dim}(H)}\rightarrow \mathbb{S}^D_H$ and vice versa.
\end{claim}
\subsection{Minimum distortion spherical subspaces}
To define principal components, we need a specific choice of distortion function $f$; see equation \eqref{eq:distortion_general}. Before presenting our choice, let us discuss previously studied cost functions.
\subsubsection{Review of existing work}
Dai and M{\"u}ller consider an intrinsic PCA for smooth functional data $\mathcal{X}$ on a Riemannian manifold $\mathcal{M}$~\cite{dai2018principal} with the distortion function $f(x) = x^2$, i.e.,
\begin{equation}\label{eq:spca_dai}
    \mathrm{cost}_{\mathrm{Dai.}}(\mathbb{S}^D_{H}| \mathcal{X}) = \mathbb{E}_N \big[ f\big( \min_{y \in \mathbb{S}^D_H}  d(x_n , y ) \big)\big].
\end{equation}
Their algorithm, Riemannian functional principal component analysis (RFPCA), is based on first solving for the base point --- the optimal zero-dimensional affine subspace, i.e., the Fr\'echet mean $p^* = \argmin_{p \in \mathbb{S}^D}\mathbb{E}_N \big[ d(x_n , p )^2\big]$. Then, they project each point to $T_{p^*}\mathcal{M}$ using the logarithmic map. Next, they perform PCA on the resulting tangent vectors to obtain the $K$-dimensional tangent subspace. Finally, they map back the projected tangent vectors to $\mathcal{M}$ (spherical space) using the exponential map. Despite its simplicity, this approach suffers from four shortcomings. (1) There is no closed-form expression for the Fr\'echet mean of spherical data. (2) Theoretical analysis on computation complexity of estimating a Fr\'echet mean is not yet known; and its computation involves an argmin operation which oftentimes cannot be easily differentiated~\cite{lou2020differentiating}. (3) Even for accurately estimated Fr\'echet mean, there is no guarantee that the optimal solution to the problem \eqref{eq:spca_dai} is the Fr\'echet mean. Huckemann and Ziezold~\cite{huckemann2006principal} show that the Fr\'echet mean may not belong to the optimal one-dimensional affine spherical subspace. (4) Even if the Fr\'echet mean happens to be the optimal base point, performing PCA in the tangent space is not the solution to problem \eqref{eq:spca_dai}.

Liu \emph{et al.} propose a spherical matrix factorization problem:
\begin{equation}\label{eq:spherical_PCA_ell2}
    \min_{ \substack{G \in \mathbb{R}^{(D+1) \times (K+1)} \\ \{  y_n \in \mathbb{S}^K \}_{n \in [N]} }} \mathbb{E}_N \big[ \| x_n - G y_n \|_{2}^2 \big]: \  G^{\top} G = I_{K+1},
\end{equation}
where $\mathcal{X} = \{x_n \in \mathbb{R}^{D+1}: n \in [N]\}$ is the measurement set and $y_1, \ldots, y_N \in \mathbb{S}^{K}$ are features in a spherical space with $C=1$~\cite{liu2019spherical}. They propose a proximal algorithm to solve for the affine subspace and features. This formalism is not a spherical PCA because the measurements do not belong to a spherical space. The objective in equation \eqref{eq:spherical_PCA_ell2} aims to best project Euclidean points to a low-dimensional spherical affine subspace \emph{with respect to the squared Euclidean distance} --- refer to \Cref{cl:spherical_alt}. Nevertheless, if we \emph{change their formalism} and let the input space be a spherical space, we arrive at:
\begin{align}
    \mathrm{cost}_{\mathrm{Liu}}(\mathbb{S}^D_{H}| \mathcal{X}) &=  \mathbb{E}_N \big[ -\cos \big( \min_{y_n \in \mathbb{S}^K }   d(x_n, G y_n) \big) \big] \nonumber \\
    &= \mathbb{E}_N \big[ f\circ d(x_n, \mathcal{P}_H(x_n) ) \big], \label{eq:spca_liu}
\end{align}
where $H$ is a tangent subspace that corresponds to $G$ (see \Cref{cl:spherical_alt}) and $\mathcal{P}_H$ is the spherical projection operator. This formalism uses distortion function $f(x)=-\mathrm{cos}(x)$.

Nested spheres~\cite{jung2012analysis} by Jung \emph{et al.} is an alternative procedure for fitting principal nested spheres to iteratively reduce the dimensionality of data. It finds the optimal $(D-1)$-dimensional subsphere $\mathcal{U}_{D-1}$ by minimizing the following cost,
\[
    \mathrm{cost}_{\mathrm{Jung}}(\mathcal{U}_{D-1} | \mathcal{X}) = \mathbb{E}_N[ (d(x_n,p) - r)^2 ],
\]
where $\mathcal{U}_{D-1} = \{ x \in \mathbb{S}^{D}: d(x, p) = r \}$ --- over $p \in \mathbb{S}^D$ and $r \in \mathbb{R}^{+}$. This is a constrained nonlinear optimization problem without closed-form solutions. Once they estimate the optimal $\mathcal{U}_{D-1}$, they map each point to the lower-dimensional spherical space $\mathbb{S}^{D-1}$ --- and repeat this process until they reach the target dimension. The subspheres are not necessarily great spheres, making this decomposition \emph{nongeodesic}.

\subsubsection{A proper cost function for spherical PCA}
In contrast to distortions $f(x)=-\mathrm{cos}(x)$ and $ f(x) = x^2$ used by Liu \emph{et al.}~\cite{liu2019spherical} and Dai and M{\"u}ller~\cite{dai2018principal}, we choose $f(x)=\mathrm{sin}^2(\sqrt{C}x)$. Using \Cref{prop:projection_and_distance_spherical}, we arrive at:
\begin{align}
    \mathrm{cost}(\mathbb{S}^D_{H}| \mathcal{X}) = \mathbb{E}_N \big[ \sum_{k \in [K^{\prime}]} \langle x_n,h^{\prime}_{k}  \rangle^2 \big]\label{eq:spherical_pca_cost_function},
\end{align}
i.e., the average $\ell_2^2$ norm of the projected points in the directions of vectors $h^{\prime}_{1}, \ldots, h^{\prime}_{K^{\prime}} \in T_p \mathbb{S}^D$. The expression \eqref{eq:spherical_pca_cost_function} leads to a tractable constrained optimization problem.
\begin{claim}\label{prob:spherical_pca_problem}
Let $x_1, \ldots, x_N \in \mathbb{S}^D$. The spherical PCA equation \eqref{eq:spherical_pca_cost_function} aims to find $p \in \mathbb{S}^D$ and orthogonal $h^{\prime}_1, \ldots, h^{\prime}_{K^{\prime}} \in T_p \mathbb{S}^D$ that minimize $\sum_{k \in [K^{\prime}]}  {h^{\prime}_{k}}^{\top} C_x h^{\prime}_{k}$ where $C_x = \mathbb{E}_N \big[ x_n x_n^{\top} \big]$.
\end{claim}
The solution to the problem in \Cref{prob:spherical_pca_problem} is the minimizer of the cost in equation \eqref{eq:spherical_pca_cost_function}, which is a set of orthogonal vectors $h^{\prime}_1, \ldots, h^{\prime}_{K^{\prime}} \in p^{\perp}$ that capture, in quantum physics terminology, the least possible energy of $C_x$. 
\begin{theorem}\label{thm:spherical_pca_problem}
Let $x_1, \ldots, x_N \in \mathbb{S}^D$. Then, an optimal solution for $p$ is the leading eigenvector of $C_x$, and $h^{\prime}_1, \ldots, h^{\prime}_{K^{\prime}}$ (basis vectors of $H^{\perp}$) are the eigenvectors that correspond to the smallest $K^{\prime}$ eigenvalues of the second-moment matrix $C_x$. 
\end{theorem}
\begin{corollary}\label{cor:spherical_pca_problem}
The optimal subspace $H$ is spanned by $K$ leading eigenvectors of $C_x$, discarding the first one.
\end{corollary}
\begin{claim}\label{cl:spherical_is_proper}
    The cost function in equation \eqref{eq:spherical_pca_cost_function} is proper.
\end{claim}
\begin{figure*}[t!] \center
\includegraphics[width=.9\textwidth]{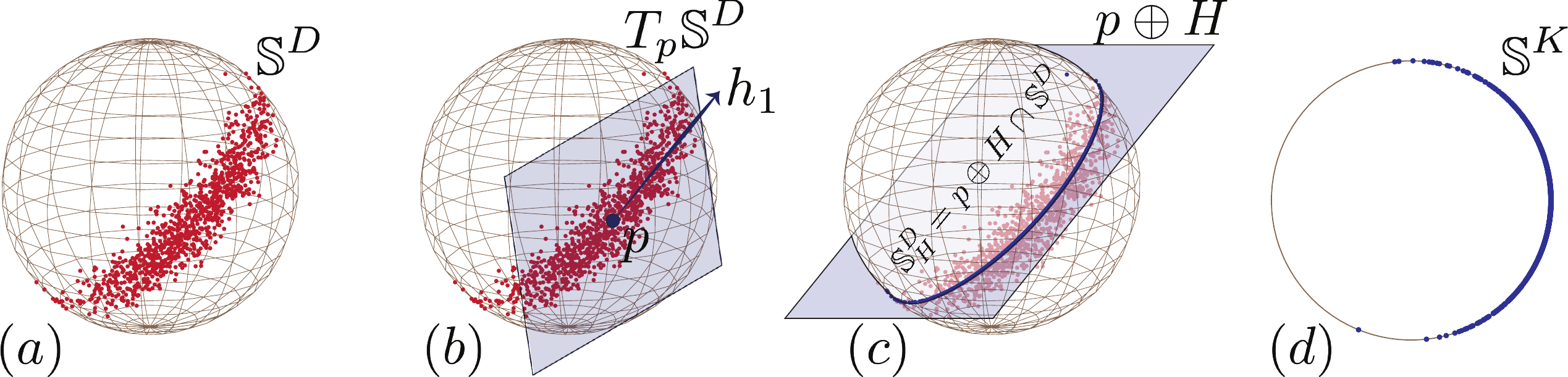}
\caption{$(a)$ A set of data points in $\mathbb{S}^D$, where $D=2$. $(b)$ The best estimate for the base point $p$ and the tangent subspace $H = h_1 \in T_p \mathbb{S}^D$ --- the spherical affine subspace $\mathbb{S}^D_H = (p \oplus H) \cap \mathbb{S}^{D}$. $(c)$ The projection of points onto $\mathbb{S}^D_H$ ($H = h_1$). $(d)$ The low-dimensional features in $\mathbb{S}^K$, where $K = \mathrm{dim}( \mathbb{S}^D_H ) = 1$.}
\label{fig:spherical_pca}
\end{figure*}
The distortion in equation \eqref{eq:spherical_pca_cost_function} implies a closed-form definition for the centroid of spherical data points, i.e., a zero-dimensional affine subspace that best represents the data.

\begin{definition}\label{def:spherical_mean}
    A spherical mean $\mu(\mathcal{X})$ for point set $\mathcal{X}$ is {\bf \textit{any point}} such that $\mathbb{E}_N \big[ f \circ  d(x_n,\mu(\mathcal{X}) )   \big] = \min_{p \in \mathbb{S}^D} \mathbb{E}_N \big[ f \circ d(x_n,p)  \big]$. The solution is a scaled leading eigenvector of $C_x$.
\end{definition}
Interpreting the optimal base point as the spherical mean in \Cref{def:spherical_mean} shows our spherical PCA has \emph{consistent centroid} and \emph{nested optimality}: \emph{optimal spherical affine subspaces of different dimensions form a chain under inclusion}. However, $\mu(\mathcal{X})$ is not unique and only identifies the {\bf direction} of the main component of points; see \Cref{fig:spherical_pca}.
\begin{remark}
     PGA~\cite{fletcher2003statistics} and       RFPCA~\cite{dai2018principal} involve the intensive task of Fréchet mean estimation. This involves iterative techniques like gradient descent or fixed-point iterations on nonlinear optimization objectives. There has been work on the numerical analysis of Fréchet mean computations~\cite{lou2020differentiating}. On the other hand, SPCA~\cite{liu2019spherical} uses alternating linearized minimization to estimate the optimal subspace. In contrast, our method (SFPCA) requires computing the second-moment matrix $C_x$ with a complexity of $O(D^2 N)$ and involves eigendecomposition with a worst-case complexity of $O(D^3)$.
\end{remark}
\section{Hyperbolic PCA} \label{sec:hyperbolic_PCA}
Let us first introduce Lorentzian spaces.
\begin{definition}
The Lorentzian $(D+1)$-space, denoted by $\mathbb{R}^{1,D}$, is a vector space equipped with the Lorentzian inner product 
\[
    \forall x,y \in \mathbb{R}^{1,D}: [x, y] = x^{\top} J_D y, \ J_D = \begin{pmatrix}
	-1 & 0^{\top}\\
	0 & I_D
	\end{pmatrix},
\]
where $I_D$ is the $D \times D$ identity matrix.
\end{definition}

The $D$-dimensional hyperbolic manifold $(\mathbb{H}^{D},g^{\mathbb{H}})$ with curvature $C<0$, where $\mathbb{H}^{D} = \{x \in \mathbb{R}^{D+1}: [x,x] = C^{-1} \ \mbox{and} \ x_0 > 0\}$ and metric $ g_p^{\mathbb{H}}(u,v) = [u,v] $ for $u, v \in T_p \mathbb{H}^D= \{ x \in \mathbb{R}^{D+1}: [x,p] = 0 \} \bydef p^{\perp}$; see~\Cref{tab:diff_geom_ingredients}.
\subsection{Eigenequation in Lorentzian spaces} 
\label{sec:Linear_Algebra_in_a_Lorentzian_Space} 
Like inner product spaces, we define operators in $\mathbb{R}^{1,D}$.
\begin{definition}\label{def:lorentzian_operators}
Let $A \in \mathbb{R}^{(D+1) \times (D+1)}$ be a matrix (operator) in $\mathbb{R}^{1,D}$. We let $A^{[\top]}$ be the $J_D$-adjoint of $A$ if and only if $A^{[\top]} = J_D A^{\top} J_D$. $A^{[-1]}$ is the $J_D$-inverse of $A$ if and only if (iff) $A^{[-1]} J_D A = A J_D A^{[-1]} = J_D$. An invertible matrix $A$ is called $J_D$-unitary iff $A^{\top} J_D A = J_D$; see~\cite{higham2003j} for more detail.
\end{definition}
The Lorentzian space $\mathbb{R}^{1,D}$ is equipped with an indefinite inner product, i.e., $\exists x \in \mathbb{R}^{1,D}: [x,x] < 0$. Therefore, it requires a form of eigenequation defined by its indefinite inner product. For completeness, we propose the following definition of eigenequation in the complex Lorentzian space $\mathbb{C}^{1,D}$.

\begin{definition}\label{def:j_eigenequation}
For $A \in \mathbb{C}^{(D+1) \times (D+1)}$, $v \in \mathbb{C}^{1,D}$ is its $J_D$-eigenvector and $\lambda$ is the corresponding $J_D$-eigenvalue if
\begin{equation}\label{eq:eigen_equation_hyperbolic}
    A J_D v =   \mathrm{sgn}([v^{*},v])\lambda v,\ \mbox{ where } \lambda \in \mathbb{C},
\end{equation}
and $v^{*}$ is the complex conjugate of $v$. The sign of the norm, $\mathrm{sgn}([v^{*},v])$, defines {\bf \emph{positive and negative $\mathbf{\mathit{J_D}}$-eigenvectors}}.
\end{definition}
\Cref{def:j_eigenequation} is subtly different from \emph{hyperbolic eigenequation}~\cite{slapnivcar1999bound} --- a special case of $(A,J)$ eigenvalue decomposition. We prefer \Cref{def:j_eigenequation} as it carries over familiar results from Euclidean to the Lorentzian space.  
\begin{proposition} \label{prop:real_j_eigenvalues}
If $A = A^{[\top]}$, then its $J_D$-eigenvalues are real.
\end{proposition}
Let $v$ be a $J_D$-eigenvector of $A$ where $|[v^{*},v]| = 1$. Then, $[v^{*}, A J_D v] = \mathrm{sgn}([v^*,v]) \lambda [v^{*},v] = \lambda$, the $J_D$-eigenvalue of $A$; see \Cref{def:j_eigenequation}. There is a connection between Euclidean and Lorentzian eigenequations. Namely, $AJ_D \in \mathbb{C}^{(D+1) \times (D+1)}$ has eigenvector $v \in \mathbb{C}^{D+1}$ and $[v^*, v] \neq 0$. Then, $\Big( |[v^{*},v]|^{-\frac{1}{2}} v, \mathrm{sgn}([v^*,v])\lambda \Big)$ is a $J_D$-eigenpair of $A$.
\begin{claim}\label{claim:eigenvector_equiv}
    $J_D$-eigenvectors of $A$ are parallel to eigenvectors of $AJ_D$.
\end{claim}
\Cref{prop:degenerate_j_eigenvectors} shows that the normalization factor is well-defined for full-rank matrices.
\begin{proposition} \label{prop:degenerate_j_eigenvectors}
If $A$ is full-rank, then $\{ v \in \mathbb{R}^{D+1} : A J_D v = \lambda v \ \text{and} \ [v^{*},v] = 0 \} = \emptyset$.
\end{proposition}
Our algorithm uses the connection between Euclidean and $J_D$-eigenequations. We extend the notion of diagonalizability to derive the optimal affine subspace; see \Cref{prop:j_diagonalizable_matrices}.
\begin{definition}\label{def:hyperbolic_eigenequation}
$A \in \mathbb{C}^{(D+1) \times (D+1)}$ is $J_D$-diagonalizable if and only if there is a $J_D$-invertible $V \in \mathbb{C}^{(D+1) \times (D+1)}$ such that $A J_D V = V J_D \Lambda$, where $\Lambda$ is a diagonal matrix.
\end{definition}

\subsection{Hyperbolic affine subspace and the projection operator}
Let $p \in \mathbb{H}^D$ and $H^{\perp}$ be a $K^{\prime}$-dimensional subspace of $T_p \mathbb{H}^D = p^{\perp}$. Following \Cref{def:affine_subspace_manifold} and \Cref{tab:diff_geom_ingredients}, we arrive at the following definition for the hyperbolic affine subspace:
\begin{equation}\label{eq:hyperbolic_subspace}
    \mathbb{H}^{D}_{H} = \{ x \in \mathbb{H}^D: [x, h^{\prime} ] = 0, \ \forall h^{\prime} \in H^{\perp} \},
\end{equation}
where $H^{\perp}$ is the orthogonal complement of $H$, i.e., $\mathbb{H}^{D}_{H} = \mathbb{H}^D \cap (p \oplus H)$. This also coincides with the metric completion of exponential barycentric hyperbolic subspace \cite{pennec2018barycentric}.
\begin{claim}\label{claim:geodesic_subhyperbolic}
    The hyperbolic subspace is a geodesic submanifold.
\end{claim}
\Cref{lem:eigenvalue_norms} shows that there is a complete set of orthogonal tangents $h_1^{\prime}, \ldots, h_{K^{\prime}}^{\prime}$ where $[h^{\prime}_{i}, h^{\prime}_{j}] = -C \delta_{i,j}$ and span $H^{\perp}$. In \Cref{prop:projection_and_distance_hyperbolic}, we provide a closed-form expression for the projection distance onto $\mathbb{H}^D_H$ in terms of the basis of $H^{\perp}$.
\begin{proposition}\label{prop:projection_and_distance_hyperbolic}
For any $\mathbb{H}_H^D$ and $x \in \mathbb{H}^D$, we have
\[
    \min_{y \in \mathbb{H}_H^D} d(x,y) = |C|^{-\frac{1}{2}} \mathrm{acosh} \Big( \sqrt{1 + \sum_{k \in [K^{\prime}]} [x, h^{\prime}_k]^2} \Big),
\]
where $\{ h^{\prime}_{k^{\prime}}\}_{k^{\prime} \in [K^{\prime}]}$ are complete orthogonal basis of $H^{\perp}$.
\end{proposition}
The projection distance monotonically increases with the norm of its residual of $x$ onto $H^{\perp}$, i.e., $\sum_{k \in [K^{\prime}]} [x,h^{\prime}_k ]^2$. \Cref{prop:projection_and_distance_hyperbolic} asks for the orthogonal basis of $H^{\perp}$---commonly, a high-dimensional space. We can use the basis of $H$ to compute the projection distance.
\begin{proposition}\label{prop:projection_and_distance_hyperbolic_parallel}
For any $\mathbb{H}_H^D$ and $x \in \mathbb{H}^D$, we have
\[
    \min_{y \in \mathbb{H}_H^D} d(x,y) = |C|^{-\frac{1}{2}} \mathrm{acosh} \Big( \sqrt{C^2 [x,p]^2 - \sum_{k \in [K]} [x,h_k]^2} \Big),
\]
where $\{ h_k\}_{k \in [K]}$ are complete orthonormal basis of $H$.
\end{proposition}
We represent points in $\mathbb{H}^D_H$ as a linear combination of the base point and tangent vectors. Given these $K+1$ vectors, we can find a low-dimensional representation for points in $\mathbb{H}^D_H$---reducing the dimensionality of hyperbolic data points.
\begin{theorem}\label{theorem:hyperbolic_affine_subspace_isometry}
The isometry $\mathcal{Q}: \mathbb{H}_H^D \rightarrow \mathbb{H}^{K}$ and its inverse are 
\[
\mathcal{Q}(x) = \alpha \begin{bmatrix}
    C[x,p] \\
    [ x,h_{1}] \\
    \vdots \\
    [ x,h_{K} ]
\end{bmatrix}, \ \mathcal{Q}^{-1} ( y ) = \alpha (-y_0 C p + \sum_{k = 1}^{K} y_k h_{k}),
\]
where $\alpha = |C|^{-\frac{1}{2}}$ and $H$ has complete orthogonal basis vectors $\{h_k\}_{k \in [K]}$. Both $\mathbb{H}_H^D$ and $\mathbb{H}^{K}$ have curvature $C < 0$.
\end{theorem}
\begin{corollary}
The affine dimension of $\mathbb{H}_H^D$ is $\mathrm{dim}(H)$.
\end{corollary}
Similar to the spherical case, we can characterize hyperbolic affine subspaces in terms of sliced $J_D$-unitary matrices---paving the way for constrained optimization methods over sliced $J_D$-unitary matrices to solve hyperbolic PCA problems.
\begin{claim}\label{claim:hyperbolic_alt}
    For any $\mathbb{H}^D_H$, there is a sliced $J_D$-unitary operator $G: \mathbb{H}^{\mathrm{dim}(H)} \rightarrow \mathbb{H}^D_H$, i.e., $G^{\top} J_D G = J_{\mathrm{dim}(H)}$, and vice versa.
\end{claim}
\subsection{Minimum distortion hyperbolic subspaces}
\label{sec:minimum_distortion_hyperbolic_subspaces}
\subsubsection{Review of existing work}
Chami \emph{et al.} propose HoroPCA \cite{chami2021horopca}. They define $\mathrm{GH}(p, q_1, \ldots, q_K)$ as the geodesic hull $\gamma_1, \ldots, \gamma_K$, where $\gamma_k$ is a geodesic such that $\gamma_k(0) = p \in \mathbb{H}^D$ and $\lim_{t \rightarrow +\infty} \gamma_k(t) = q_k \in \partial \mathbb{H}^D$ for all $k \in [K]$. The geodesic hull of $\gamma_1, \ldots, \gamma_K$ contains straight lines between $\gamma_k(t)$ and $\gamma_{k^{\prime}}(t^{\prime})$ for all $t, t^{\prime} \in \mathbb{R}$ and $k, k^{\prime} \in [K]$.
\begin{claim}\label{claim:chami}
$\mathrm{GH}(p, q_1, \ldots, q_K)$ is a hyperbolic affine subspace.  
\end{claim}
Their goal is to maximize a proxy for the projected variance:
\begin{equation}\label{eq:chami_cost}
    \mathrm{cost}_{\mathrm{Chami}} (\mathbb{H}^D_H| \mathcal{X} ) = -
    \mathbb{E}_N \Big[ d\big(\widehat{\mathcal{P}}_H(x_n) , \widehat{\mathcal{P}}_H(x_{n^{\prime}}) \big)^2\Big],
\end{equation}
where $\widehat{\mathcal{P}}_H$ is the horospherical projection operator --- which is not a geodesic (distance-minimizing) projection. They propose a sequential algorithm to minimize the cost in equation \eqref{eq:chami_cost}, using a gradient descent method, as follows: (1) the base point is computed as the Frechet mean via gradient descent; and (2) a higher-dimensional affine subspace is estimated based on the optimal affine subspace of lower dimension. One may formulate the hyperbolic PCA problem as follows:
\begin{equation}\label{eq:hyperbolic_PCA_ell2}
    \min_{\substack{G \in \mathbb{R}^{(D+1) \times (K+1)} \\ \{y_n \in \mathbb{H}^{K}\}_{n \in [N]} }} \mathbb{E}_N \big[ \| x_n - G y_n \|_{2}^2 \big]: \  G^{\top} J_D G = J_K,
\end{equation}
where $x_1,\ldots, x_N \in \mathbb{R}^{D+1}$ are the measurements and $y_1, \ldots, y_N \in \mathbb{H}^{K}$ are low-dimensional hyperbolic features. The formulation in equation \eqref{eq:hyperbolic_PCA_ell2} leads to the decomposition of a Euclidean matrix in terms of a sliced-$J_D$ unitary matrix and a hyperbolic matrix --- a topic for future studies.
\subsubsection{A Proper Cost Function for Hyperbolic PCA} \label{sec:proper_cost_function_hyperbolic_PCA}
We choose $f(x)=\mathrm{sinh}^2(\sqrt{|C|}x)$ to arrive at the following cost:
\begin{align}
    \mathrm{cost}(\mathbb{H}^D_{H}| \mathcal{X}) =   \mathbb{E}_N \big[ \sum_{k \in [K^{\prime}]} [x_n,h^{\prime}_k]^2 \big];\label{eq:hyperbolic_pca_cost_function}
\end{align}
see \Cref{prop:projection_and_distance_hyperbolic}. We interpret $\mathrm{cost}(\mathbb{H}^D_{H}| \mathcal{X})$ as the \emph{aggregate dissipated image} of the points in directions of normal tangent vectors. If $x \in \mathbb{H}^D$ has no components in the direction of normal tangents --- i.e., $[x , h_k^{\prime}] = 0$ where $h^{\prime}_1,\ldots, h^{\prime}_{K^{\prime}}$ are orthogonal basis vectors for $H^{\perp}$ --- then $ \sum_{k \in [K^{\prime}]} [x,h^{\prime}_{k}]^2 = 0$. Our distortion function in equation \eqref{eq:hyperbolic_pca_cost_function} leads to the formulation of hyperbolic PCA as a constrained optimization problem:
\begin{problem}\label{prob:hyperbolic_pca_problem}
Let $x_1, \ldots, x_N \in \mathbb{H}^D$ and $C_x = \mathbb{E}_N \big[ x_n x_n^{\top}\big]$. The hyperbolic PCA aims to find a point $p \in \mathbb{H}^D$ and a set of orthogonal vectors $h^{\prime}_1, \ldots, h^{\prime}_{K^{\prime}} \in T_p \mathbb{H}^D = p^{\perp}$  that minimize the function $\sum_{k \in [K^{\prime}]}  {h^{\prime}_{k}}^{\top} J_d C_x J_d h^{\prime}_{k}$.
\end{problem}
\Cref{prob:hyperbolic_pca_problem} aims to minimize the cost in equation \eqref{eq:hyperbolic_pca_cost_function}, i.e.,
\begin{align*}
     \mathrm{cost}(\mathbb{H}^D_{H}| \mathcal{X}) =   \sum_{k \in [K^{\prime}]}  {h^{\prime}_{k}}^{\top} J_D C_x J_D h^{\prime}_{k},
\end{align*}
where $p \in \mathbb{H}^D, h^{\prime}_1, \ldots, h^{\prime}_{K^{\prime}} \in p^{\perp}, [h^{\prime}_{i}, h^{\prime}_{j}] = -C\delta_{i,j}$ and $C_x =\mathbb{E}_N \big[ x_n x_n^{\top}\big]$, over $p \in \mathbb{H}^D, H \subseteq p^{\perp} , \mathrm{dim}(H) = K$. \Cref{prob:hyperbolic_pca_problem} asks for $J_D$-orthogonal vectors $h^{\prime}_1, \ldots, h^{\prime}_{K^{\prime}}$ in an appropriate tangent space $T_p \mathbb{H}^D$ that capture the least possible energy of $C_x$ with respect to the Lorentzian scalar product. 
\begin{remark}
    The spectrum of a matrix is its set of eigenvalues. Discarding an eigenvalue from the matrix's eigenvalue decomposition approximates the matrix with an error proportional to the magnitude of the discarded eigenvalue, that is, discarded energy. Similarly, one can define the \( J_D \)-spectrum of the second-moment matrix \( C_x \) as the set of its \( J_D \)-eigenvalues. As we demonstrate in numerical experiments, we use \( J_D \)-spectrum to identify the existence of outlier hyperbolic points.
\end{remark}
This is akin to the Euclidean PCA: the subspace is spanned by the leading eigenvectors of the covariance matrix. However, to prove the hyperbolic PCA theorem, we need a technical result on $J_D$-diagonalizability of the second-moment matrix.
\begin{proposition}\label{prop:j_diagonalizable_matrices}
If $A \in \mathbb{R}^{(D+1) \times (D+1)}$ is a symmetric and $J_D$-diagonalizable matrix, i.e., $A J_D V = V J_D \Lambda$, that has distinct (in absolute values) diagonal elements of $\Lambda$, then $A = V \Lambda V^{\top}$ where $V$ is a $J_D$-unitary matrix.
\end{proposition}
From \Cref{prop:j_diagonalizable_matrices}, any symmetric, $J_D$-diagonalizable matrix with distinct (absolute) $J_D$-eigenvalues has $D$ positive and one negative $J_D$-eigenvectors --- all orthogonal to each other.

\begin{theorem}\label{thm:hyperbolic_pca_problem}
Let $x_1, \ldots, x_N \in \mathbb{H}^D$ and $C_x = \mathbb{E}_N \big[ x_n x_n^{\top}\big]$ be a $J_D$-diagonalizable matrix. Then, the optimal solution for point $p$ is the scaled negative $J_D$-eigenvector of $C_x$ and the optimal $h^{\prime}_1, \ldots, h^{\prime}_{K^{\prime}}$ are the scaled positive $J_D$-eigenvectors that correspond to the smallest $K^{\prime}$ $J_D$-eigenvalues of $C_x$. And $H$ is spanned by $K=D-K^{\prime}$ scaled positive $J_D$-eigenvectors that correspond to the leading $J_D$-eigenvalues of $C_x$.
\end{theorem}
The $J_D$-diagonalizability condition requires $C_x$ to be similar to a diagonal matrix. \Cref{prop:j_diagonalizable_matrices} provides a sufficient condition for its $J_D$-diagonalizability; in fact, we conjecture that symmetry is sufficient even if it has repeated $J_D$-eigenvalues.
\begin{claim}\label{cl:hyperbolic_is_proper}
    The cost function in equation \eqref{eq:hyperbolic_pca_cost_function} is proper.
\end{claim}
The proper cost in equation \eqref{eq:hyperbolic_pca_cost_function} implies the following closed-form definition for the hyperbolic centroid.
\begin{definition}\label{def:hyperbolic_mean}
    A hyperbolic mean of $\mathcal{X}$ is $\mu(\mathcal{X})$ if $\mathbb{E}_N \big[ f\circ  d(x_n, \mu(\mathcal{X}))  \big] = \min_{p \in \mathbb{H}^D} \mathbb{E}_N \big[ f \circ d(x_n, p)  \big]$. The solution is the scaled negative $J_D$-eigenvector of $C_x$.
\end{definition}
\begin{remark}
    The formalism of space form (Euclidean, spherical, and hyperbolic) PCAs shows similarities through the use of (in)definite eigenequations. This arises from the introduction of proper cost functions which resulted in quadratic cost functions with respect to the base points and tangent vectors. However, this approach is not necessarily generalizable to other Riemannian manifolds. This limitation is due to the absence of (1) a simple Riemannian metric, (2) a closed-form distance function, and (3) closed-form exponential and logarithmic maps in general Riemannian manifolds, e.g., the manifold of rank-deficient positive semidefinite matrices \cite{lahav2023procrustes}.
\end{remark}
\section{NUMERICAL RESULTS}
We compare our \emph{space form PCA algorithm} (SFPCA) to other leading algorithms in terms of accuracy and speed. 
\subsection{Synthetic data and experimental setup}
We generate random, noise-contaminated points on known (but random) spherical and hyperbolic affine subspaces. We then apply PCA methods to recover the projected points after estimating the affine subspace. We conduct experiments examining the impact of number of points $N$, the ambient dimension $D$, the dimension of the affine subspace $K$, and the noise level $\sigma$ on the performance of algorithms.

\subsubsection{Random affine subspace} 
For fixed ambient and subspace dimensions, $D$ and $K$, we sample from a normal distribution and normalize it to get the spherical (or hyperbolic)  point $p$. We then generate random vectors from the standard normal distribution and use the Gram–Schmidt process to construct tangent vectors: project the first random vector onto $p^{\perp}$ and normalize it to $h_1 \in T_p S^{D}$, where $S \in \{ \mathbb{S}, \mathbb{H}\}$. We then project the second random vector onto $(p \oplus H)^{\perp}$, where $H = h_1$, and normalize it to $h_2 \in T_p S^{D}$. We repeat this until we form a $K$-dimensional affine subspace in $T_p S^D$.

\subsubsection{Noise-contaminated points}
\begin{figure*}[t!]
\centering
\includegraphics[width=.9\textwidth]{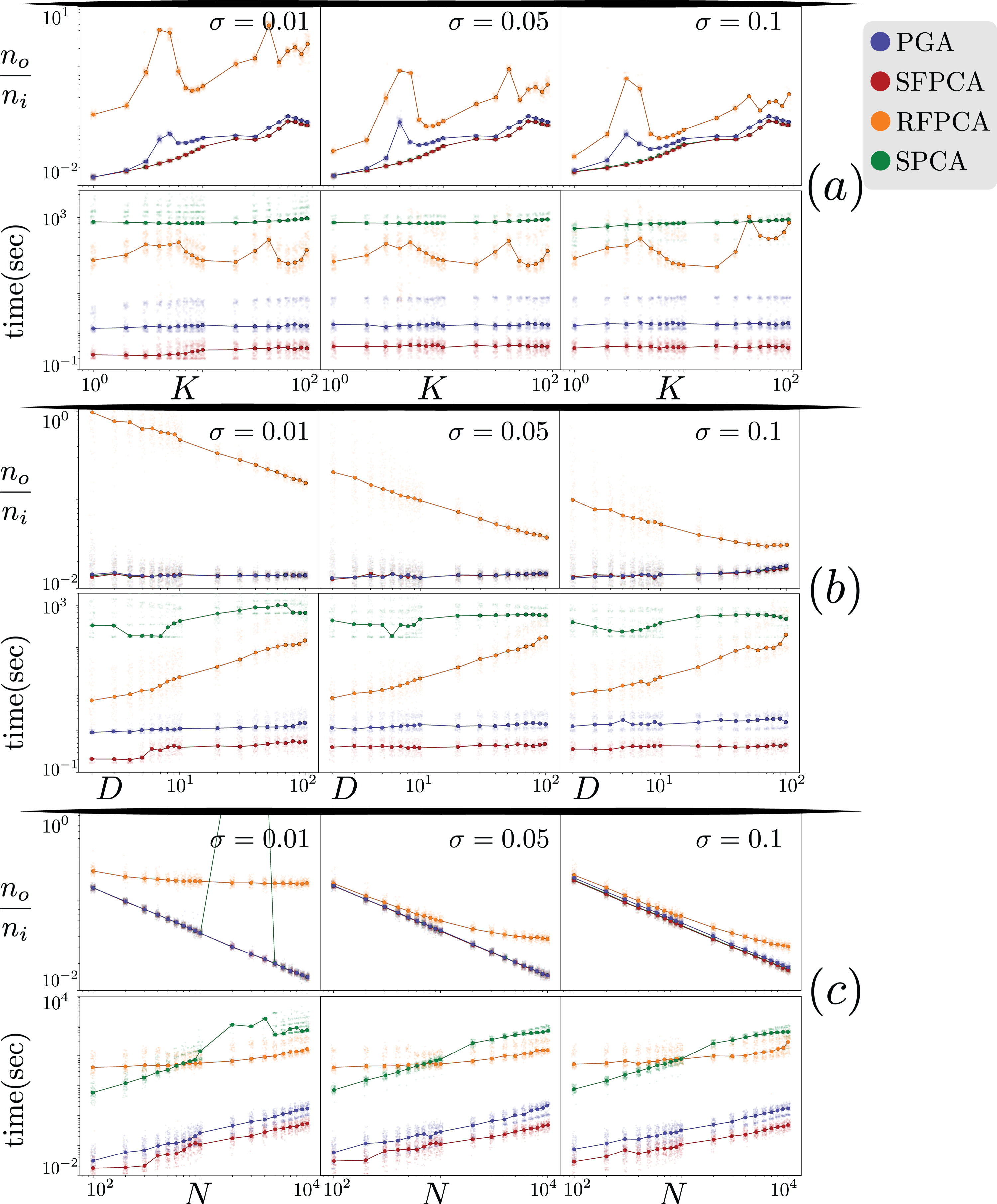}
\caption{For each spherical experiment, on the y-axes, we report running time and normalized output error. A dot corresponds to a random trial, and connected circles show the median across all trials. Figures $(a, b, c)$ show the results for $\mathbb{S}(K_1), \mathbb{S}(D_1), \mathbb{S}(N_1)$, respectively. All axes are in logarithmic scale.} 
\label{fig:spca_experiment}
\end{figure*}

\begin{figure}[b]
\center
\includegraphics[width=.49\textwidth]{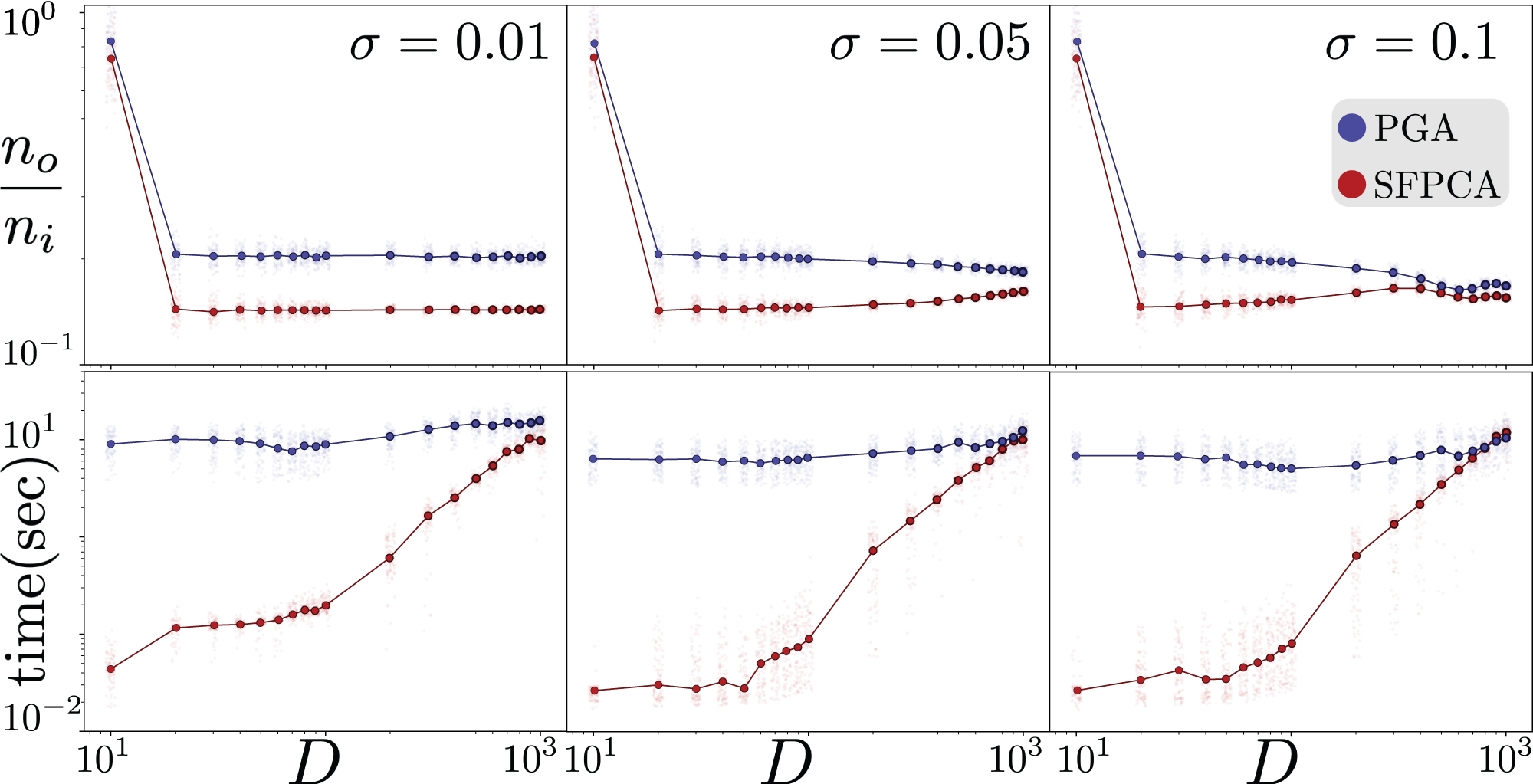}
\caption{Spherical experiment $\mathbb{S}(D_2)$. The y-axes show running time and normalized output error. All axes are in logarithmic scale.}
\label{fig:spca_experiment2}
\end{figure}
Let $H \in \mathbb{R}^{(D+1) \times K}$ be the subspace of $T_p S^D$. For $n \in [N]$, we generate $c_n \sim \mathcal{N}(0, \alpha_S I_{K+1})$ and let $v_n = Hc_n$. To add noise, we project $\nu_n \sim \mathcal{N}(0,  \alpha_S \sigma^2I_{D+1})$ onto $T_p S^D$, i.e, $p^{\perp} v_n$. We then let $x_n = \mathrm{exp}_p(v_n + p^{\perp}\nu_n )$ be the noise-contaminated point. Finally, $\alpha_S = \frac{\pi}{4}$ if $S = \mathbb{S}$ and $\alpha_S = 1$ if $S = \mathbb{H}$.
\subsubsection{PCA on noisy data} 
We use each algorithm to estimate an affine subspace $S^{D}_{\widehat{H}}$ where $\widehat{H} \subseteq T_{\widehat{p}} S^D$ and $\widehat{p}$ are the estimated parameters. We let $n_i \bydef \mathbb{E}_N[ d(x_n, \mathcal{P}_{H}(x_n)) ]$ be the empirical mean of the distance between measurements and the true subspace, and $n_o \bydef \mathbb{E}_N \Big[ d\big( \mathcal{P}_{\widehat{H}}(x_n), \mathcal{P}_H(\mathcal{P}_{\widehat{H}}(x_n)) \big) \Big]$ be the average distance between denoised points $\{ \mathcal{P}_{\widehat{H}}(x_n) \}_{n \in [N]}$ and the true affine subspace. If $S^D_{\widehat{H}}$ is a good approximation to $S^D_{H}$, then $n_o$ is small. 
%When $n_o > n_i$, the algorithm has failed to estimate a meaningful affine subspace. 
We evaluate the performance of algorithms using the normalized output error, $\frac{n_o}{n_i}$.
\begin{remark}
    The ratio of $n_o$ over $n_i$ quantifies how much the estimated points are farther from the true subspace compared to the original noise-contaminated points. This is a normalized quantity, i.e., it is invariant with respect to the scale of data points, which makes it ideal for comparing results as $D$, $K$, $\sigma$, and $N$ vary. A reasonable upper bound for this ratio is $1$ --- as PCA is expected to denoise the point sets by finding the optimal low-dimensional affine subspace for them.
\end{remark}
\subsubsection{Randomized Experiments and Algorithms} 
For each random affine subspace and noise-contaminated point set, we report the normalized error and the running time for each algorithm. Then, we repeat each random trial $100$ times. We use our implementation of principal geodesic analysis (PGA) \cite{fletcher2003statistics}. We also implement Riemannian functional principal component analysis (RFPCA) for spherical PCA \cite{dai2018principal} and spherical principal component analysis (SPCA) \cite{liu2019spherical}. Since SPCA is computationally expensive, we first run our SFPCA to provide it with good initial parameters. For hyperbolic experiments, we use HoroPCA \cite{chami2021horopca} and barycentric subspace analysis (BSA) \cite{pennec2018barycentric}, implemented by Chami \emph{et al.}~\cite{chami2021horopca}.

\subsection{Spherical PCA}\label{sec:exp-sph}
\subsubsection{Experiment $\mathbb{S}(K_1)$} 
For a fixed $D = 10^{2}, N = 10^{4}$, increasing the subspace dimension $K$ worsens the normalized output errors for all algorithms; see \Cref{fig:spca_experiment} $(a)$. RFPCA is unreliable, while other methods are similar in their error reduction pattern. When $K$ is close to $D$, SFPCA has a marginal but consistent advantage over PGA. SFPCA is faster than the rest, and $K$ has a minor impact on running times.
\subsubsection{Experiments $\mathbb{S}(D_1)$ and $\mathbb{S}(D_2)$} In $\mathbb{S}(D_1)$ --- fixed $K = 1$, $N = 10^{4}$ and varying $D$ --- PGA, SFPCA, and SPCA exhibit a similar denoising performance, not impacted by $D$; see \Cref{fig:spca_experiment} $(b)$. RFPCA has higher output error levels than other methods. To further compare SFPCA and its close competitor PGA, we design the challenging experiment $\mathbb{S}(D_2)$ with $K = 10$ and $N = 10^{3}$. In this setting, SFPCA exhibits a clear advantage over PGA in error reduction; see \Cref{fig:spca_experiment2}. In both settings, SFPCA continues to be the fastest in almost all conditions despite using a warm start for PGA.

\subsubsection{Experiment $\mathbb{S}(N_1)$} 
For fixed $K=1$ and $D = 10^2$, when we change $N$ and $\sigma$, our SFPCA has the fastest running time; and it is tied in having the lowest normalized output error with SPCA and PGA; see \Cref{fig:spca_experiment} $(c)$. As expected, increasing $N$ generally makes all methods slower, partially because the computation of $C_x$ has $O(N)$ complexity. Computing a base point $p$ using iterative computations on all $N$ points is time-consuming with $N$, whereas our SFPCA has worst-case  complexity of $O(D^3)$. SFPCA provides similar error reductions compared to the rest due to providing an excessive number of points to each algorithm. SPCA fails in some cases, as evident from the erratic behavior of normalized output error. SPCA takes about $15$ minutes on $10^4$ points in each trial, while our SFPCA takes less than a second. PGA is the closest competitor in normalized error but is about three times slower.

\begin{figure*}[t]
\centering
\includegraphics[width=0.9\textwidth]{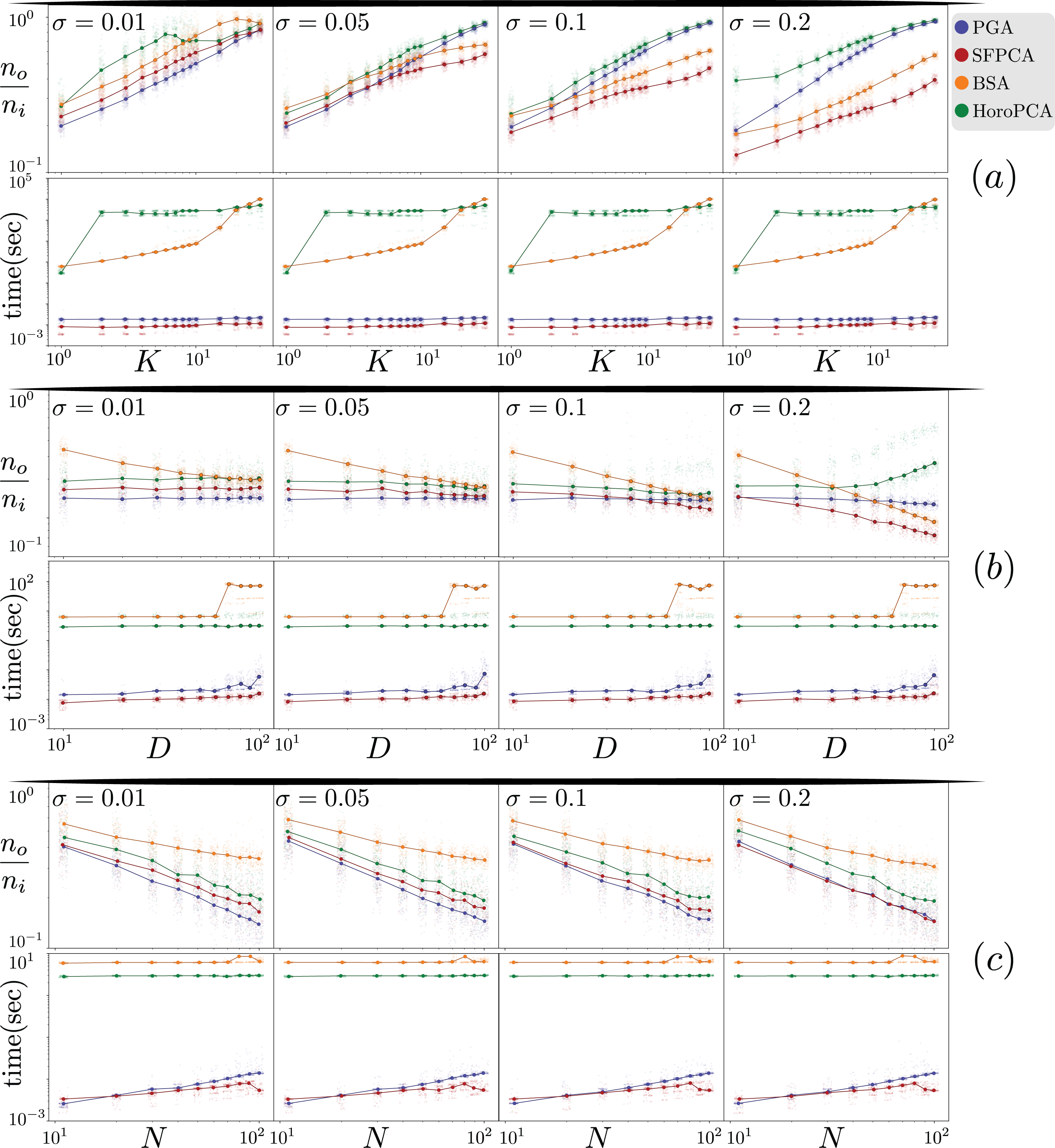}
\caption{For each scaled-down hyperbolic experiment, on the y-axes, we report running time and normalized output error in logarithmic scale. A dot corresponds to a random trial, and circles show the median across all trials. Figures in rows $(a), (b)$, and $(c)$ are $\mathbb{H}(K_1), \mathbb{H}(D_1)$, and $\mathbb{H}(N_1)$.}
\label{fig:hpca_experiment_full}
\end{figure*}

\begin{figure*}[t!]
\centering
\includegraphics[width=0.9\textwidth]
{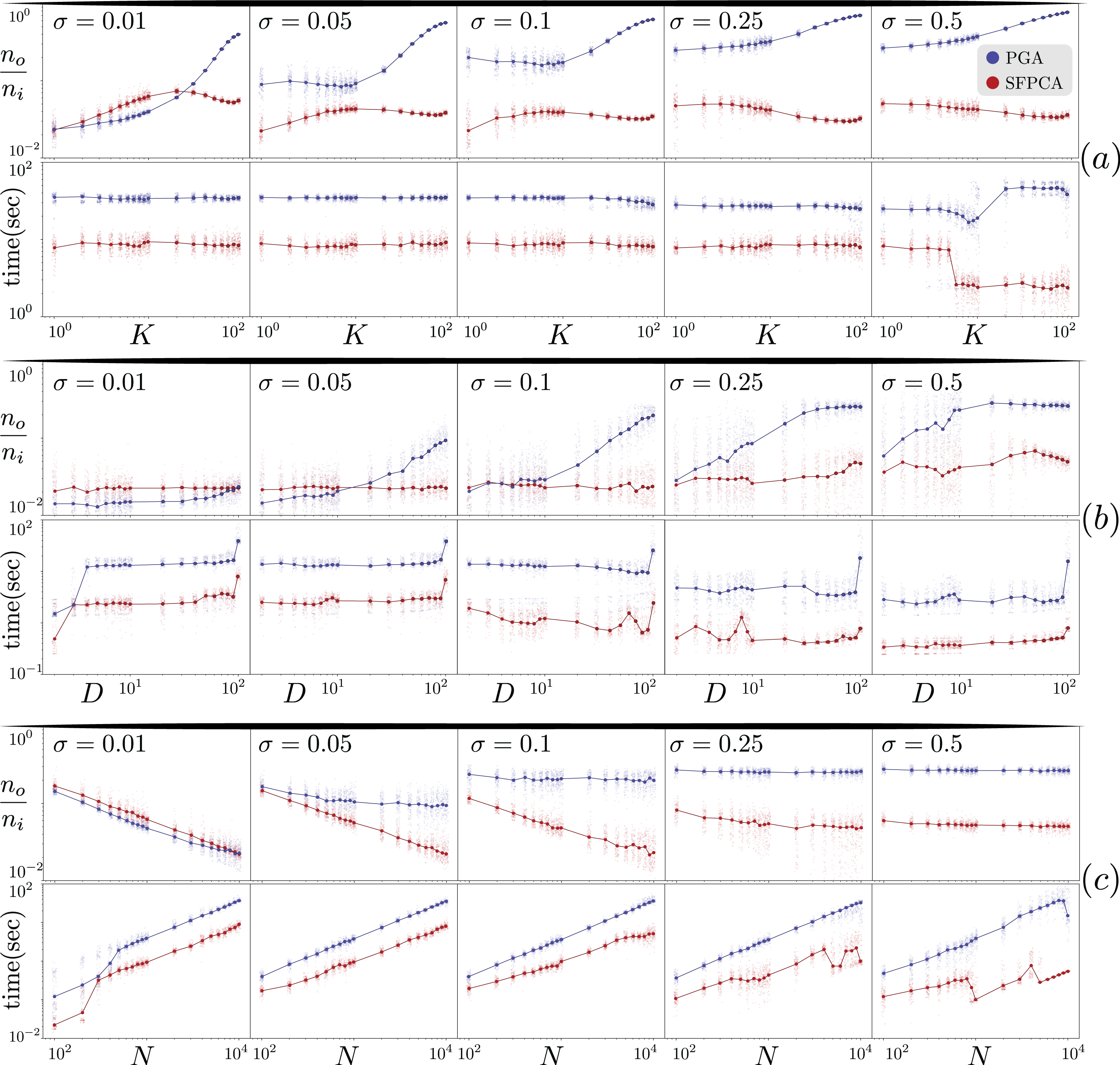}
\caption{For each full-scale hyperbolic experiment, on the y-axes, we report running time and normalized output error in logarithmic scale. A dot corresponds to a random trial, and circles show the median. Figures in rows $(a), (b)$, and $(c)$ are $\mathbb{H}(K_2), \mathbb{H}(D_2)$, and $\mathbb{H}(N_2)$. All axes are in logarithmic scale.
}\label{fig:hpca_experiment_large}
\end{figure*} 
\subsection{Hyperbolic PCA}\label{sec:exp-hyp}
\subsubsection{Experiments $\mathbb{H}(K_1)$ and $\mathbb{H}(K_2)$}
On small datasets in $\mathbb{H}(K_1)$ ($D = 50$, $N = 51$), for each trial, HoroPCA and BSA take close to an hour whereas SFPCA and PGA take milliseconds; see \Cref{fig:hpca_experiment_full} $(a)$. Increasing $K$ only increases the running time of BSA and HoroPCA but does not change SFPCA's and PGA's. This is expected as they estimates an affine subspace greedily one dimension at a time. Regarding error reduction, as expected, all methods become less effective as $K$ grows. For small $\sigma$, all methods achieve similar normalized output error levels with only a slight advantage for PGA and SFPCA. As $\sigma$ increases, PGA and HoroPCA become less effective compared to BSA and SFPCA. For large $\sigma$, SFPCA exhibits a clear advantage over all other methods. In the larger $\mathbb{H}(K_2)$ experiments ($D = 10^2$, $N = 10^4$), we compare the two fastest methods, SFPCA and PGA; see \Cref{fig:hpca_experiment_large} $(a)$. When $\sigma$ is small, both methods have similar denoising performance for small $K$; SFPCA performs better only for larger $K$. As $\sigma$ increases, SFPCA outperforms PGA irrespective of $K$.

\subsubsection{Experiments $\mathbb{H}(D_1)$ and $\mathbb{H}(D_2)$}
In $\mathbb{H}(D_1)$, we fix $K = 1$, $N = 101$ and in $\mathbb{H}(D_2)$, we let $K = 1$, $N = 10^4$. Changing $D$ impacts each method differently; see \Cref{fig:hpca_experiment_full} $(b)$. Both SFPCA and PGA take successively more time as $D$ increases, but they remain significantly faster than the other two, with average running times below $0.1$ second. The running time of HoroPCA is (almost) agnostic to $D$ since its cost function (projected variance) is free of the parameter $D$. Neither HoroPCA nor BSA outperform SFPCA in error reduction. All methods improve in their error reduction performance as $D$ increases. For large $\sigma$, SFPCA provides the best error reduction performance among all algorithms. Comparing the fastest methods SFPCA and PGA, we observe consistent patterns in $\mathbb{H}(D_1)$ and $\mathbb{H}(D_2)$: (1) SFPCA is faster, regardless of $D$ and the gap between the two methods can be as high as a factor of $10$. (2) When $\sigma < 0.1$, PGA slightly outperforms SFPCA in reducing error; with the lowest noise ($\sigma = 0.01$), PGA gives $17\%$ better accuracy, in average over all values of $D$. However, as $\sigma$ increases, SFPCA becomes more effective; at the highest end ($\sigma = 0.5$), SFPCA outperforms PGA by $40\%$, in average over $D$; see \Cref{fig:hpca_experiment_large} $(b)$.
\subsubsection{Experiments $\mathbb{H}(N_1)$ and $\mathbb{H}(N_2)$}
In $\mathbb{H}(N_1)$, ($K=1$, $D=10$), increasing $N$ impacts the running time of SFPCA and PGA due to computing $C_x$; see \Cref{fig:hpca_experiment_full} $(c)$. Nevertheless, both are orders of magnitude faster than HoroPCA and BSA. All methods provide improved error reduction as $N$ increases. Comparing the fast methods SFPCA and PGA on larger datasets $\mathbb{H}(N_2)$ ($K=1$, $D=10^2$) shows that SFPCA is always faster, has a slight disadvantage in output error on small $\sigma$, and substantial improvements on large $\sigma$; see \Cref{fig:hpca_experiment_large} $(c)$.
\subsection{Real Data: Spherical Spaces}
We evaluate the performance of PCA methods using following datasets: (1) \textsc{Intestinal Microbiome}: Lahti et al. \cite{lahti2014tipping} analyzed the gut microbiota of $N = 1,006$ adults covering $D=130$ bacterial groups. The study explored the effects of age, nationality, BMI, and DNA extraction methods on the microbiome. They assessed variations in microbiome compositions across young ($18$–$40$), middle-aged ($41$–$60$), and older ($61$–$77$) age groups (a ternary classification problem). (2) \textsc{Throat Microbiome}: Charlson et al. \cite{charlson2010disordered} investigated the impact of cigarette smoking on the airway microbiota in $29$ smokers and $33$ non-smokers (a binary classification problem) using culture-independent $454$ pyrosequencing of 16S rRNA. (3) \textsc{Newsgroups}: Using Python's $\texttt{scikit-learn}$ package, the $20$ newsgroups dataset was streamlined to a binary classification problem by retaining $N=400$ samples from two distinct classes. Feature reduction was performed using TF-IDF, narrowing it down to $D= 3000$ features to improve computational efficiency. Each dataset has undergone standard preprocessing, e.g., normalization and square-root transformation, to ensure the data points are spherically distributed. For a fixed subspace dimension, we estimate spherical affine subspaces. Then, we compute the projected spherical data points and denoise the original compositional data.

\emph{Distortion Analysis.} For compositional data, we calculate distance matrices using Aitchison ($\mathbb{AI}$), Jensen--Shannon ($\mathbb{JS}$), and total variation ($\mathbb{TV}$) metrics. We also compute the spherical distance matrix ($\mathbb{S}$). For each embedding dimension $K$, we compute projected point sets. We then compute the normalized errors; an example of normalized error is $\frac{\| D_{\mathbb{TV}} - \widehat{D}_{\mathbb{TV}} \|_F}{\| D_{\mathbb{TV}} \|_F}$ where $D_{\mathbb{TV}}$ and $\widehat{D}_{\mathbb{TV}}$ are total variation distance matrices for the original and estimated data. For each algorithm, we then divide these normalized errors by their average across all algorithms, providing relative measures, that is, if the resulting relative error is greater than $1$, the algorithm performs worse than average. We then report the mean and standard deviation of relative errors across different dimensions; see  \Cref{tab:spherical_table}. On all datasets, SFPCA outperforms the rest. \textsc{Newsgroups} experiments are limited to SFPCA and PGA due to the significant computational complexity of SPCA and RFPCA.

\emph{Classification Performance.} For each $K$, using the denoised compositional data, we perform two classifiers: a five-layer neural network ($\mathbb{NN}$) and a random forest model ($\mathbb{RF}$). We normalize and report the mean and standard deviation of classification accuracies by the average accuracy of all methods. From \Cref{tab:spherical_table}, SFPCA outperforms competing methods on \textsc{Intestinal Microbiome} and \textsc{Newsgroups}, though the accuracy differences are mostly less than one percent. In \Cref{sec:more_exps}, we further compare the performance of the two classifiers on \textsc{Newsgroups} as it relates to PCA analysis.

\begin{table*}[t]
    \centering
    \renewcommand{\arraystretch}{1.5}
    \setlength\tabcolsep{2pt}
    \caption{The mean and standard deviation of normalized distance errors are divided by their average across methods. Classification accuracies are percentage deviations from $100\%$ --- representing the average accuracy across methods. Boldface and red indicate SFPCA and the top-performing method. Lower distortions $(\downarrow)$ and higher accuracies $(\uparrow)$ are better.}
    \begin{tabular}{|c|cccc|cccc|cc|}
\hline
\multirow{3}{*}[0pt]{\raisebox{4ex}{\parbox{1cm}{\centering Metric \\ (Method)}}} & \multicolumn{4}{c|}{\textsc{Throat Microbiome}} & %
    \multicolumn{4}{c|}{\textsc{Intestinal Microbiome}} &\multicolumn{2}{c|}{\textsc{Newsgroups}}\\
\cline{2-11}
 & SFPCA & RFPCA & SPCA & PGA &SFPCA & RFPCA & SPCA & PGA & SFPCA & PGA \\
\hline
\hline
$\mathbb{S}(\downarrow)$ & \textcolor{darkred}{$\boldsymbol{0.88 \pm 0.99}$} & $1.13 \pm 1.28$ & $0.9 \pm 0.99$ & $1.1 \pm 1.18$ &\textcolor{darkred}{$\boldsymbol{0.75 \pm 1.64}$} & $0.93 \pm 2.09$ & $1.47 \pm 2.07$& $0.85 \pm 1.79$&  \textcolor{darkred}{$\boldsymbol{0.77 \pm 1.03}$} &  $1.23 \pm 1.4$\\
\hline 
\hline
 $\mathbb{AI}(\downarrow)$ & \textcolor{darkred}{$\boldsymbol{0.98 \pm 0.48}$} & $1.03 \pm 0.57$ &$0.99 \pm 0.46$ & $1.01 \pm 0.52$ & \textcolor{darkred}{$\boldsymbol{0.8 \pm 1.06}$} & $0.81 \pm 1.03$ & $1.55 \pm 1.04$ & $0.84 \pm 1.14$ & \textcolor{darkred}{$\boldsymbol{0.999 \pm 0.4}$} &  $1.001 \pm 0.5$ \\
\hline
$\mathbb{JS}(\downarrow)$ & \textcolor{darkred}{$\boldsymbol{0.91 \pm 0.84}$} & $1.08 \pm 1.04$ & $0.95 \pm 0.84$ & $1.06 \pm 0.97$ & \textcolor{darkred}{$\boldsymbol{0.75 \pm 1.55}$} & $0.93 \pm 1.96$& $1.47 \pm 1.93$& $0.85 \pm 1.7$& \textcolor{darkred}{$\boldsymbol{0.9 \pm 0.81}$} &   $1.1 \pm 0.98$ \\
\hline
$\mathbb{TV}(\downarrow)$ & \textcolor{darkred}{$\boldsymbol{0.91 \pm 1.0}$} & $1.08 \pm 1.21$ & $0.94 \pm 1.0$ & $1.07 \pm 1.14$ & \textcolor{darkred}{$\boldsymbol{0.76 \pm 1.65}$} & $0.92 \pm 2.09$ & $1.49 \pm 2.07$& $0.83 \pm 1.75$ & \textcolor{darkred}{$\boldsymbol{0.87 \pm 0.88}$} &  $1.13 \pm 1.12$ \\
\hline \hline
$(\mathbb{NN})(\uparrow)$ & $\boldsymbol{-0.06 \pm 3.03}$  & $0.15 \pm 3.15$  & $-0.26 \pm 3.7$ & \textcolor{darkred}{$0.17 \pm 3.0$} &\textcolor{darkred}{$\boldsymbol{0.04 \pm 0.36}$} & $-0.05 \pm 0.9$ & $-0.02 \pm 0.6$ & $0.03 \pm 0.3$ &  \textcolor{darkred}{$\boldsymbol{0.04 \pm 1.6}$} &  $-0.04 \pm 1.6$ \\
\hline
$(\mathbb{RF})(\uparrow)$ & $\boldsymbol{0.59 \pm 9.5}$ & \textcolor{darkred}{$1.2 \pm 9.6$} & $-1.9 \pm 9.3$ &  $0.11 \pm 9.8$ &\textcolor{darkred}{$\boldsymbol{0.4 \pm 2.3}$} & $0.11 \pm 2.7$ & $-0.74 \pm 2.8$ & $0.2 \pm 2.3$ &  \textcolor{darkred}{$\boldsymbol{0.3 \pm 3.5}$} &  $-0.3 \pm 4.4$ \\
\hline
% etc. ...
\end{tabular}
    \label{tab:spherical_table}
\end{table*}
\subsection{Real Data: Hyperbolic Spaces}
We use a biological dataset of $103$ plant and algal transcriptomes \cite{1kp-pilot}. The authors inferred phylogenetic trees from genome-wide genes. Tree leaves are present-day species, internal nodes are ancestral species, and branch weights represent evolutionary distances. The dataset includes an \say{unfiltered} version with $852$ trees and a \say{filtered} version with $844$ trees after removing error-prone genes and filtering problematic sequences. Errors appear as outliers, with more expected in the unfiltered dataset {\em a priori}. Other studies \cite{TreeShrink,PhylteR} have used this dataset to evaluate outlier detection methods. We preprocess each tree by rescaling branch lengths to a diameter of $10$, compute the distance matrix between leaves, and embed it into a $D$-dimensional ($D = 20$) hyperbolic space using a semidefinite program~\cite{tabaghi2020hyperbolic}. We use two metrics to evaluate PCA results, and then apply them for outlier detection.

\emph{Distortion Analysis.} 
For a fixed dimension $K$, we estimate hyperbolic affine subspaces, compute the projected hyperbolic points, and their hyperbolic distance matrix $\widehat{D}_{\mathbb{H}}$. The normalized distance error $\frac{\|D_\mathbb{H}-\widehat{D}_\mathbb{H}\|_F}{\|D_\mathbb{H}\|_F}$ is calculated, where $D_\mathbb{H}$ is the original distance matrix. These errors are averaged over $K \in [D]$ for each algorithm and then divided by the average normalized errors across all algorithms --- providing relative errors. If the relative error is greater than $1$, the algorithm performs worse than average. For each algorithm, we report the mean and standard deviation of these relative errors across all gene trees. Distortion is not a perfect measure of PCA accuracy as highly noise-contaminated data should experience high distortions during the projection (denoising) step. In all experiments, PGA outperforms others in terms of distortion (\Cref{tab:hyperbolic_table}). SFPCA provides an average distance-preserving performance, contrary to synthetic experiments. We conjecture this may be due to the trees being relatively small (see scaled-down hyperbolic experiments in~\Cref{fig:hpca_experiment_full}), or to high noise levels making distortion an inappropriate accuracy metric, or the discordance between our choice of distortion function $f = \mathrm{cosh}$ (which overemphasizes large distances) and the distance distortion metric.

\emph{Quartet Scores.} To use a biologically motivated accuracy measure, we use the quartet score \cite{firstquartet}. For a target dimension $K$, each algorithm is applied to an embedded hyperbolic point set to compute the projected (denoised) points. For each set of four points, we find the \emph{optimal} tree topology with minimum distance distortion using the four-point condition \cite{Warnow2017}. For $10^5$ randomly chosen (but fixed) sets of four projected points, we estimate their topology and compare it with the \emph{true topology} from the gene trees. For each dimension $K \in [D]$, we compute the percentage of correctly estimated topologies, then average this over all dimensions. We normalize and report the mean and standard deviation of quartet scores by the average score of all methods, as detailed in \Cref{tab:hyperbolic_table}. In these experiments, PGA and SFPCA exhibit the best performance compared to the alternatives. This is particularly informative as the quartet score measures tree topology accuracy, not distance.

\begin{table}[t]
    \centering
    \renewcommand{\arraystretch}{1.5}
    \setlength\tabcolsep{2pt}
    \caption{Normalized distance errors ($\mathbb{H}$) mean and standard deviation are divided by their average across methods. Quartet scores ($\mathbb{Q}$) are percentage deviations from $100\%$ (the average). Lower distortions $(\downarrow)$ and higher scores $(\uparrow)$ are better.}
    \begin{tabular}{|c|cc|cc|}
\hline
Method & \multicolumn{2}{c|}{\textsc{Filtered}} & %
    \multicolumn{2}{c|}{\textsc{Unfiltered}} \\
\cline{2-5}
 & $\mathbb{Q}(\uparrow)$ & $\mathbb{H}(\downarrow)$ & $\mathbb{Q}(\uparrow)$  & $\mathbb{H}(\downarrow)$   \\
\hline
\hline
SFPCA  & \textcolor{darkred}{$\boldsymbol{1.50 \pm 1.91}$} &   $\boldsymbol{0.98 \pm 0.23}$& $\boldsymbol{1.10 \pm 1.74}$ & $\boldsymbol{1.01 \pm 0.25}$ \\
\hline
PGA & $1.48 \pm 1.47$ & \textcolor{darkred}{$0.55 \pm 0.09$} & \textcolor{darkred}{$1.80 \pm 1.45$} & \textcolor{darkred}{$0.53 \pm 0.08$} \\
\hline
BSA & $1.48 \pm 1.61$ & $1.45 \pm 0.24$  & $1.67 \pm 1.91$ & $1.45 \pm 0.31$ \\
\hline
HoroPCA & $-4.48 \pm 2.79$  & $1.01 \pm 0.21$ & $-4.57 \pm 2.58$& $1.00 \pm 0.20$  \\
\hline 
\end{tabular}
    \label{tab:hyperbolic_table}
\end{table}

\subsubsection{Outlier Detection with Hyperbolic Spectrum}
\begin{figure*}[h]
\centering
\includegraphics[width=.9\textwidth]{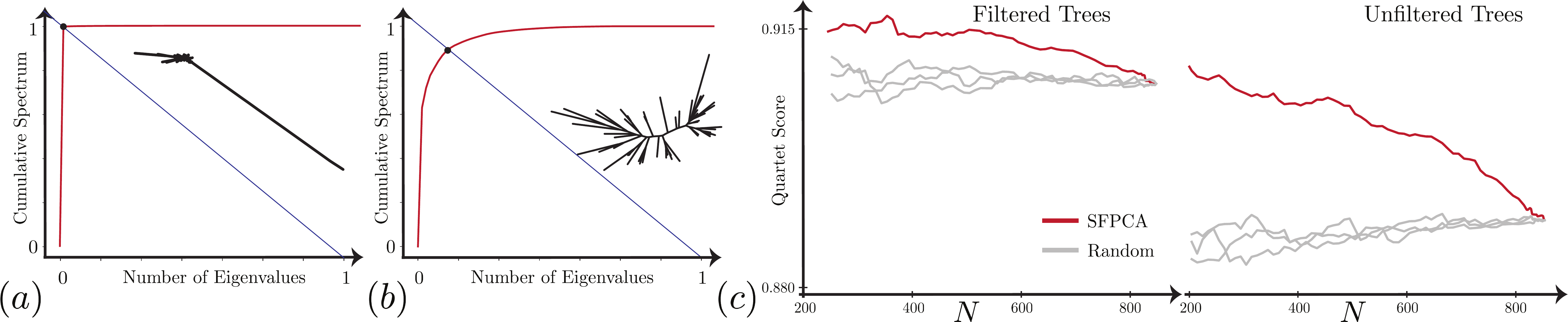}
\caption{$(a)$ and $(b)$: Normalized retained energy versus normalized number of eigenvalues for two gene trees. The knee point (intersection with $y = 1-x$ line) for trees with outliers approaches $(0,1)$. $(c)$ The quartet score for species trees constructed using the top $N$ trees (knee values) versus random orders.}
\label{fig:astral_results}
\end{figure*}
We showcase the practical utility of the $J_D$-eigenequation (\Cref{def:j_eigenequation}) in species tree estimation. As proved in \Cref{thm:hyperbolic_pca_problem}, the principal axes align with the leading $J_D$-eigenvectors of $C_x$. Thus, the optimal SFPCA cost corresponds to the sum of its neglected $J_D$-eigenvalues. We conjecture that a tree with outliers has more outlier $J_D$-eigenvalues; see \Cref{fig:astral_results} $(a,b)$. If a tree has an outlier set of species (likely from incorrect sequences), its second leading $J_D$-eigenvalue ($\lambda_2$)\footnote{The $J_D$-eigenvalue $\lambda_1$ corresponds to the base point.} is significantly larger than the rest. We quantify this by plotting its normalized retained energy $\frac{\sum_{k=2}^{K}\lambda_k}{\sum_{d=2}^{D} \lambda_d}$ (or cumulative spectrum) versus the normalized embedding dimension (or number of $J_D$-eigenvalues) $x = K/D$ and finding its the knee point. This lets us \emph{sort} gene trees by their hyperbolic spectrum.

After sorting, we use the top $N$ trees with the smallest knee values (least prone to outliers) to construct a species topology using ASTRAL \cite{astral3}. ASTRAL outputs the quartet score between the estimated species tree and the input gene trees, where a higher score indicates more congruence among input trees. Thus, a higher score after filtering means outlier gene trees, likely inferred from problematic sequences, have been removed. Our results (\Cref{fig:astral_results} $(c)$) show that hyperbolic spectrum-based sorting --- offered only by our SFPCA --- effectively identifies the worst trees most dissimilar to others, without explicitly comparing tree topologies. In contrast, random sorting keeps the quartet score fixed. Filtered trees have a higher score than unfiltered trees and benefit less from further filtering. It is remarkable that using eigenvalues alone, we can effectively find genes with discordant evolutionary histories.

\section*{Acknowledgment}
The authors would like to thank the National Institutes of Health (NIH) for their financial support. This research was partially funded by NIH Grant 1R35GM142725.
\section*{Appendix}
\setcounter{section}{0}
\renewcommand{\thesection}{\Alph{section}}
\subsubsection{\Cref{claim:geodesic_submanifold_spherical}}
    Let $x,y \in \mathbb{S}^D_{H} \subseteq \mathbb{S}^D$ and $\gamma_{x,y}(t)$ be the geodesic where $\gamma_{x,y}(0) = x$ and $ \gamma_{x,y}(1) = y$. This geodesic lies at the intersection of $\mathrm{span}\{x,y\}$ and $\mathbb{S}^D$. Since $x, y \in \mathbb{S}^D_H = p \oplus H \cap \mathbb{S}^D$ and $p \oplus H$ is a subspace, we have $\mathrm{span}\{x,y\} \subseteq p \oplus H$. Therefore, we have $\gamma_{x,y}(t) \in \mathrm{span}\{x,y\} \cap \mathbb{S}^D \subseteq p \oplus H \cap \mathbb{S}^D = \mathbb{S}^D_H$ for all $t \in [0, 1]$.
\subsubsection{\Cref{cl:spherical_alt}} 
    For a spherical affine subspace $\mathbb{S}^D_H$, we define the sliced unitary matrix $G=\begin{bmatrix} C^{\frac{1}{2}}p, C^{-\frac{1}{2}}h_1, \ldots, C^{-\frac{1}{2}}h_K  \end{bmatrix}$ where $p$ is the base point and $h_1, \ldots, h_{K}$ are orthogonal basis for $H$. Then, we have $Gy \in \mathbb{S}^D_H$ for all $y \in \mathbb{S}^K$. Conversely, for any sliced unitary matrix $G =\begin{bmatrix} g_0, g_1, \ldots, g_K  \end{bmatrix}$, we let $p = C^{-\frac{1}{2}} g_0$ and $h_k = C^{\frac{1}{2}}g_k$ for $k \in [K]$. Since $h_k$'s and $p$ are orthogonal, we can define the spherical affine subspace.
\subsubsection{\Cref{prob:spherical_pca_problem}}
    We simplify its PCA cost as $\mathrm{cost}(\mathbb{S}^D_{H}| \mathcal{X}) = 1 - \mathbb{E}_N \big[  \|P_{H}(x_n) \|_2^2 \big] \stackrel{\text{(a)}}{=}  \mathbb{E}_N [\sum_{k \in [K^{\prime}]} \langle x_n,h^{\prime}_{k} \rangle^2]  \stackrel{\text{(b)}}{=} \sum_{k \in [K^{\prime}]}  {h^{\prime}_{k}}^{\top} C_x h^{\prime}_{k}$ where $\text{(a)}$ follows from \Cref{prop:projection_and_distance_spherical}, $h_1^{\prime}, \ldots, h_{K^{\prime}}^{\prime}$ are orthogonal basis for $H^{\perp} \subseteq T_p \mathbb{S}^D$, and $\text{(b)}$ follows from cyclic property of trace. 
\subsubsection{\Cref{cl:spherical_is_proper}}
    From \Cref{cor:spherical_pca_problem} and \Cref{def:spherical_mean}, an optimal zero-dimensional affine subspace (a point) is a subset of any other spherical affine subspace. In general, $\mathbb{S}^{D}_{H_1}  \subseteq \mathbb{S}^{D}_{H_2} \ \ \mbox{if and only if} \ \ \mathrm{dim}(\mathbb{S}^{D}_{H_1}) \leq \mathrm{dim}(\mathbb{S}^{D}_{H_2})$.
\subsubsection{\Cref{claim:eigenvector_equiv}}
    Let $v \in \mathbb{C}^{D+1}$ be such that $AJ_D v = \mathrm{sgn}([v^*,v]) \lambda v$ where $|[v^*,v]| = 1$, then $\| v \|_2^{-1} v$ is an eigenvector of $AJ_D$ with eigenvalue of $\mathrm{sgn}([v^*,v]) \lambda$. 
\subsubsection{\Cref{claim:geodesic_subhyperbolic}}
    Let $x,y \in \mathbb{H}^D_{H}$ and $\gamma_{x,y}$ be the geodesic where $\gamma_{x,y}(0) = x, \gamma_{x,y}(1) =y$. This geodesic belongs to $\mathrm{span}\{x,y\} \cap \mathbb{H}^D$. We have $\mathrm{span} \{ x,y \} \subseteq p \oplus H$ since $x, y \in \mathbb{H}^D_H =  p \oplus H \cap \mathbb{H}^D$ and $p \oplus H$ is a subspace. Thus, we have $\gamma_{x,y}(t) \in \mathrm{span}\{x,y\} \cap \mathbb{H}^D \subseteq \mathbb{H}^D_H$ for all $t \in [0, 1]$.
\subsubsection{\Cref{claim:hyperbolic_alt}}
    For a hyperbolic affine subspace $\mathbb{H}^D_H$, we define $G=\begin{bmatrix} |C|^{\frac{1}{2}}p, |C|^{-\frac{1}{2}}h_1, \ldots, |C|^{-\frac{1}{2}}h_K  \end{bmatrix}$ where $p$ is the base point and $h_1, \ldots, h_{K}$ are orthogonal basis for $H$. We have $G^{\top}J_K G = J_K$ and $Gy \in \mathbb{H}^D$ for all $y \in \mathbb{H}^D_H$. Conversely, for any sliced unitary matrix $G =\begin{bmatrix} g_0, g_1, \ldots, g_K  \end{bmatrix}$, we let $p =  |C|^{-\frac{1}{2}} g_0$ and $h_k = |C|^{-\frac{1}{2}}g_k$ for $k \in [K]$. Since $h_k$'s and $p$ are orthogonal, we can define the spherical affine subspace.
\subsubsection{\Cref{claim:chami}}
Let $p \in \mathbb{H}^D$, $q_1,\ldots, q_K \in \partial H^{D}$, and $\gamma_1, \ldots, \gamma_K$ be the aforementioned geodesics. A point $x \in S \bydef \mathrm{GH}(p, q_1, \ldots, q_K)$ belongs to a geodesic whose end points are $\gamma_k(t)$ and $\gamma_{k^{\prime}}(t^{\prime})$ for $t, t^{\prime} \in \mathbb{R}$ and $k,k^{\prime} \in [K]$. Let us show $x \in \mathbb{H}^{D}_H$ for a subspace $H \subseteq T_p \mathbb{H}^D$. From \Cref{claim:geodesic_subhyperbolic}, $\mathbb{H}^{D}_H$ is a geodesic submanifold. It suffices to show that $\gamma_1, \ldots, \gamma_K$ belong to $\mathbb{H}^{D}_H$. Let $h_k = \gamma^{\prime}_k(0) \in T_p \mathbb{H}^D$, for all $k \in [K]$, and $H = \mathrm{span}\{h_1, \ldots, h_K \}$. This proves $S \subseteq \mathbb{H}^D_H$. Conversely, let $x \in \mathbb{H}^D_H$ --- the hyperbolic affine subspace constructed as before. Since $\mathbb{H}^D_H$ is a geodesic submanifold, $x$ belongs to a geodesic whose end points are $\gamma_k(t)$ and $\gamma_{k^{\prime}}(t^{\prime})$ for $t, t^{\prime} \in \mathbb{R}$ and $k,k^{\prime} \in [K]$, constructed as before. From $x \in \mathrm{convhull}(\gamma_1, \ldots, \gamma_K)$, we have $\mathbb{H}^D_H \subseteq S$. 
\subsubsection{\Cref{cl:hyperbolic_is_proper}}
From \Cref{thm:hyperbolic_pca_problem} and \Cref{def:hyperbolic_mean}, a optimal zero-dimensional affine subspace is a subset of any other affine subspace. For optimal affine subspaces $\{ \mathbb{H}^{D}_{H_i} \}$, we have $\mathbb{H}^{D}_{H_1}  \subseteq \mathbb{H}^{D}_{H_2}$ if and only if $\mathrm{dim}(\mathbb{H}^{D}_{H_1}) \leq \mathrm{dim}(\mathbb{H}^{D}_{H_2})$.    

\subsubsection{\Cref{prop:projection_and_distance_spherical}}
Consider the following Lagrangian:
\begin{equation}\label{eq:lagrangian_spherical_projection}
    \mathcal{L}(y, \gamma, \Lambda)  = \langle x,y \rangle + \gamma ( \langle y,y \rangle - C^{-1})  + \sum_{k \in [K^{\prime}]} \lambda_{k} \langle y, h^{\prime}_{k} \rangle ,
\end{equation}
where $\Lambda \bydef \{ \lambda_{k}:  k\in [K^{\prime}] \}$, $\langle y,y \rangle =C^{-1}$ and $\langle y, h^{\prime}_{k} \rangle = 0$, $\forall k \in [K^{\prime}]$. The solution to equation \eqref{eq:lagrangian_spherical_projection} takes the form $\mathcal{P}_{H}(x) = \sum_{i \in [K^{\prime}]} \alpha_{i} h^{\prime}_{i} + \beta x$, for scalars $\{ \alpha_{i} \}_{i \in [K^{\prime}]}$ and $\beta$. The subspace conditions --- $\langle \mathcal{P}_{H}(x), h^{\prime}_k \rangle=0$, $\forall k \in [K^{\prime}]$ --- give $\mathcal{P}_{H}(x) = \beta P_H(x)$ where $P_H(x) = x- C^{-1}\sum_{k \in [K^{\prime}]} \langle x,h^{\prime}_{k} \rangle h^{\prime}_{k}$. Enforcing the norm condition, we arrive at $\mathcal{P}_{H}(x) = C^{-\frac{1}{2}}\|P_H(x) \|_2^{-1} P_H(x)$ where $\|P_H(x) \|_2^2 = C^{-1}(1 - \sum_{k \in [K^{\prime}]} \langle x,h^{\prime}_k \rangle^2)$. Then, we have $d(x, \mathcal{P}_{H}(x) ) = C^{-\frac{1}{2}} \mathrm{acos} \big(  C\langle x,C^{-\frac{1}{2}}\|P_H(x) \|_2^{-1} P_H(x) \big) = C^{-\frac{1}{2}}\mathrm{acos}\big(  C^{\frac{1}{2}} \| P_H(x)\|_2 \big)$. If $P_H(x) = 0$ --- i.e., $x \in \mathrm{span}(H^{\perp})$ --- then $\mathcal{P}_H(x) \in \mathbb{S}^D_H$ is nonunique; but the projection distance well-defined as $d(x,y) =C^{-\frac{1}{2}} \frac{\pi}{2}$, $\forall y \in \mathbb{S}_H^D$.
\subsubsection{\Cref{prop:projection_and_distance_spherical_parallel}}
Let $A = [Cp, h_1, \ldots, h_{K}, h_1^{\prime}, \ldots, h_{K^{\prime}}^{\prime}] \in \mathbb{R}^{(D+1) \times (D+1)}$ where $K+K^{\prime} = D$. Distinct columns of $A$ are orthogonal --- i.e., $A^{\top} A = C I_{D+1}$. Hence, $p, h_1, \ldots, h_{K}, h^{\prime}_1, \ldots, h^{\prime}_{K^{\prime}}$ are linearly independent. Therefore, we have $P_{H}(x) = \sum_{k \in [K]} \alpha_k h_k + \sum_{k \in [K^{\prime}]} \beta_{k} h^{\prime}_{k} + \gamma p \stackrel{\text{(a)}}{=} x - C^{-1}\sum_{k  \in [K^{\prime}]} \langle x, h^{\prime}_{k} \rangle h^{\prime}_{k}$ where $\{ \alpha_k \}_{k \in [K]}$, $\{ \beta_{k} \}_{k\in [K^{\prime}]}$, $\gamma$ are scalars, and $\text{(a)}$ is due to \Cref{prop:projection_and_distance_spherical}. We hence have $\beta_{k}  = C^{-1}\langle P(x), h^{\prime}_{k} \rangle = 0$, $\alpha_{k}  = C^{-1}\langle P(x), h_{k} \rangle = C^{-1}\langle x, h_{k} \rangle$ and $\gamma = C\langle P(x), p\rangle = C\langle x,p \rangle$. We can accordingly compute $\| P_{H}(x) \|_2$, and prove the proposition.
\subsubsection{\Cref{prop:real_j_eigenvalues}}
    Let $A$ be a real matrix such that $A = A^{[\top]}$, that is, $A = J_D A^{\top} J_D$. Let $(v,\lambda)$ be an eigenvector-eigenvalue pair of $A$. Then, we have $A J_D v =  \mathrm{sgn}\big([v^*,v]\big) \lambda v$ and $A J_D v^{*} =  \mathrm{sgn}\big([v^*,v]\big) \lambda^{*} v^{*}$. We have $\lambda^* = \lambda$ since
    \begin{align*}
        \lambda [v^*, v] &= (v^{*})^{\top} J_D A J_D v \ \mathrm{sgn}\big([v^*,v]\big) \\
        &=  (AJ_D v^{*})^{\top} J_D v \ \mathrm{sgn}\big([v^*,v]\big) \\
        &= \mathrm{sgn}\big([v^*,v]\big) \lambda^{*} v^{H} J_D v \ \mathrm{sgn}\big([v^*,v]\big) = \lambda^{*} [v^{*} ,v] .
    \end{align*}
\subsubsection{\Cref{prop:degenerate_j_eigenvectors}}
Let $A$ be a full-rank matrix and  $v \in \mathbb{C}^{1,D}$ such that $A J_D v = \lambda v$ and $[v^{*},v] = 0$. Then, we have $v^{H} J_D A J_D v = v^{H} J_D \lambda v = \lambda [v^*, v]= 0$. This contradicts with the assumption that $A$ is full-rank.
\subsubsection{\Cref{prop:projection_and_distance_hyperbolic}}
    Consider the following Lagrangian:
    \begin{equation}\label{eq:lagrangian_hyperbolic_projection}
        \mathcal{L}(y,\gamma, \Lambda)  = [x,y] + \gamma ( [y,y] - C^{-1})  + \sum_{k \in [K^{\prime}]} \lambda_{k} [y, h^{\prime}_{k}],
    \end{equation}
    where $\Lambda= \{ \lambda_{k} \}_{k \in [K^{\prime}]}$, admits the solution $\mathcal{P}_{H}(x) = \sum_{i \in [K^{\prime}]} \alpha_{i} h^{\prime}_{i} + \beta x$, for scalars $\{ \alpha_{i} \}_{i \in [K^{\prime}]}$ and $\beta$. The subspace conditions, $[ \mathcal{P}_{H}(x), h^{\prime}_k ]=0, \forall k \in [K^{\prime}]$, give  $\mathcal{P}_{H}(x) = \beta P_H(x)$ where $P_H(x) = x + C^{-1}\sum_{k \in [K^{\prime}]} [x,h^{\prime}_k]h^{\prime}_k$. Enforcing the norm condition, we get $\mathcal{P}_{H}(x) = |C|^{-\frac{1}{2}}\|P_H(x) \|^{-1} P_H(x)$, $\|P_H(x) \| \bydef \sqrt{-[P_H(x), P_H(x)]} = |C|^{-\frac{1}{2}}\sqrt{1 +\sum_{k \in [K^{\prime}]} [x,h^{\prime}_k]^2}$. We have $d(x, \mathcal{P}_{H}(x) ) = |C|^{-\frac{1}{2}} \mathrm{acosh} \big(  C[x,\mathcal{P}_{H}(x) ] \big) = |C|^{-\frac{1}{2}}\mathrm{acosh}\big( |C|^{\frac{1}{2}} \| P_H(x)\| \big)$.

\subsubsection{\Cref{prop:projection_and_distance_hyperbolic_parallel}}
    Let $A  = [Cp, h_1, \ldots, h_{K}, h_1^{\prime}, \ldots, h_{K^{\prime}}^{\prime}] \in \mathbb{R}^{(D+1) \times (D+1)}$ where $K+K^{\prime} = D$. Columns of $A$ are $J$-orthogonal, that is, $A^{\top} J_D A = |C|J_D$. Since we have $J_D A^{\top} J_D A_D = |C| I_{D+1}$, $p, h_1, \ldots, h_{K}, h^{\prime}_1, \ldots, h^{\prime}_{K^{\prime}}$ are linearly independent, we have $P_{H}(x) = \sum_{k \in [K]} \alpha_k h_k + \sum_{k \in [K^{\prime}]} \beta_{k} h^{\prime}_{k} + \gamma p \stackrel{\text{(a)}}{=} x +C^{-1} \sum_{k \in [K^{\prime}]} [x, h^{\prime}_{k}] h^{\prime}_{k}$ where $\{ \alpha_k \}_{k \in [K]}$, $\{ \beta_{k} \}_{k \in [K^{\prime}]}$, $\gamma$ are scalars, and $\text{(a)}$ is due \Cref{prop:projection_and_distance_hyperbolic}. So we have $\beta_{k}  = -C^{-1}[P(x), h^{\prime}_{k}] = 0$, $\alpha_{k}  = -C^{-1}[x, h_{k}]$, and $\gamma = C[P(x), p] = C[x,p]$, i.e., $P_{H}(x) = C[x,p]p-C^{-1} \sum_{k \in [K]}[x, h_{k}]h_k$. Then, we can compute $\| P_H(x)\|$ and prove the proposition.
\subsubsection{\Cref{prop:j_diagonalizable_matrices}}
Let $A =V \Lambda V^{\top}$ where $V$ is $J_D$-unitary. Since $\Lambda$ is a diagonal matrix, we let $A = V J_D \Lambda J_D V^{\top}$. We have $A J_D V = V J_D \Lambda J_D (V^{\top} J_D V) = V J_D \Lambda$, i.e., $A$ is $J_D$-diagonalizable. 

Let $A \in \mathbb{R}^{(D+1) \times (D+1)}$ be symmetric and $A J_D V = V J_D \Lambda$ for a $J_D$-invertible $V$ and $\Lambda = \mathrm{diag}(\lambda_d)_{d \in [D+1]}$ with distinct (in absolute values) elements. Then, we have
\begin{equation}\label{eq:eigen_equation_hyperbolic_2}
A J_D v_d = 
  \begin{cases}
    -\lambda_d v_d, & \text{ if } d = 1 \\
    \lambda_d v_d,  & \text{ if } d \neq 1,
  \end{cases}    
\end{equation}
where $v_d$ be the $d$-th column of $V$. In the eigenequation \eqref{eq:eigen_equation_hyperbolic_2}, 
the negative (positive) signs are designated for the eigenvectors with negative (positive) norms. For distinct $i,j$, we have
\begin{align*}
    \big| \lambda_i [v_i, v_j] \big| &\stackrel{\text{(a)}}{=} \big| [A J_D v_i, v_j] \big|  = \big|  v_i^{\top} J_D A^{\top} J_D v_j \big|  \\
    &= \big|   v_i^{\top} J_D A J_D v_j \big|  = \big|  [v_i,  A J_D v_j] \big|   =  \big|  \lambda_{j} [v_i,  v_j] \big|,
\end{align*}
where $\text{(a)}$ is due to the eigenequation \eqref{eq:eigen_equation_hyperbolic_2}. Since $|\lambda_i| \neq |\lambda_{j}|$, then we must have $[v_i,v_j] = 0$. Without loss of generality, $J_D$-eigenvectors are scaled such that $|[v_d,v_d] |= 1$ for $d \in [D+1]$. \Cref{lem:eigenvalue_norms} shows that $[v_1, v_1] = -1$ and $[v_d, v_d] = 1$ for $d > 1$.
\begin{lemma}\label{lem:eigenvalue_norms}
    $V^{\top} J_D V = J_D$.
\end{lemma}
\begin{proof}
Let $A$ where $A J_D V =V J_D \Lambda$ and $V$ $J_D$-diagonalizes $A^{\top} J_D A$, viz., $V^{\top} J_D (A^{\top} J_D A) J_D V \Lambda \stackrel{\text{(a)}}{=} (V J_D \Lambda)^{\top} J_D (V J_D \Lambda) \stackrel{\text{(b)}}{=} \Lambda J_D \Lambda^{\prime} J_D \Lambda = \Lambda^2 \Lambda^{\prime}$, where $\text{(a)}$ follows from $A J_D V =V J_D \Lambda$,  $\text{(b)}$ from $[v_i,v_j] = 0$ for $i \neq j$, and $\Lambda^{\prime}$ is a diagonal matrix. This shows that $V$ diagonalizes $B \bydef (A J_D)^{\top} J_D (A J_D)$, i.e.,  $V^{\top} (A J_D)^{\top} J_D (A J_D) V$ is a diagonal matrix. However, $B$ is a symmetric matrix with only one negative eigenvalue~\cite{tabaghi2020hyperbolic}. Therefore, without loss of generality, the first diagonal element of $\Lambda^{\prime}$ is negative.
\end{proof}

From \Cref{lem:eigenvalue_norms}, we have $V^{-1} = J_D V^{\top} J_D$ and $A = A J_D V (J_D V)^{-1} = V J_D \Lambda V^{-1} J_D 
    = V J_D \Lambda  (J_D V^{\top} J_D) J_D= V J_D \Lambda J_D V^{\top} = V \Lambda V^{\top}$.

\subsubsection{\Cref{thm:spherical_affine_subspace_isometry}}
Consider $\mathbb{S}^D_H$ with orthogonal tangents $h_1, \ldots, h_{K}$. For $x \in \mathbb{S}^D_H$, we have $P_H(x) = x$, $\| P_H(x) \|_2^2 = C^{-1}$, and $ \| \mathcal{Q}(x) \|_2^2 = C^{-1}$, i.e., $\mathcal{Q}(x) \in \mathbb{S}^K$ and $\mathcal{Q}$ is a map between $\mathbb{S}_H^D$ to $\mathbb{S}^K$; see proof of \Cref{prop:projection_and_distance_spherical_parallel}. We also have $x = P_H(x) = C\langle x,p \rangle p + C^{-1}\sum_{k \in [K]} \langle x,h_k \rangle h_k = \mathcal{Q}^{-1} \circ \mathcal{Q} (x)$ for all $x \in \mathbb{S}^D_H$. Hence, $\mathcal{Q}^{-1}$ is the inverse map of $\mathcal{Q}$ --- a bijection. Finally, $\mathcal{Q}$ is an isometry between $\mathbb{S}_H^D$ and $\mathbb{S}^K$ since $d(x_1,x_2) = d(\mathcal{Q}(x_1),\mathcal{Q}(x_2))$ for all $x_1, x_2 \in \mathbb{S}^D_H$.
\subsubsection{\Cref{thm:spherical_pca_problem}}
We have $\sum_{k \in [K^{\prime}]}  {h^{\prime}_{k}}^{\top} C_x h^{\prime}_{k} \geq \sum_{k \in [K^{\prime}]} C \lambda_{D+1-k}(C_x)$ where $\lambda_d(C_x)$ is the $d$-the largest eigenvalue of $C_x$. We achieve the lower bound if we let $h^{\prime}_{k} = C^{\frac{1}{2}}v_{D+1-k}(C_x)$ for $k \in [K^{\prime}]$. The optimal base point is any vector in $\mathrm{span}\{ h^{\prime}_1, \ldots, h^{\prime}_{K^{\prime}}\}^{\perp}$ with norm $C^{-\frac{1}{2}}$, i.e., $p = C^{-\frac{1}{2}}v_1(C_x)$ allows for nested affine subspaces.

\subsubsection{\Cref{theorem:hyperbolic_affine_subspace_isometry}}
Consider $\mathbb{H}^D_H$ with orthogonal tangents $h_1, \ldots, h_{K}$. For $x \in \mathbb{H}^D_H$, we have $P_H(x) = x$, $[ P_H(x),P_H(x) ] = C^{-1}$, and $[ \mathcal{Q}(x),\mathcal{Q}(x)] = C^{-1}$, i.e., $\mathcal{Q}(x) \in \mathbb{H}^K$ and $\mathcal{Q}$ is a map between $\mathbb{H}_H^D$ to $\mathbb{H}^K$; see proof of \Cref{prop:projection_and_distance_hyperbolic_parallel}. We also have $x = P_H(x) = C[ x,p ] p - C^{-1}\sum_{k \in [K]} [ x,h_k ] h_k = \mathcal{Q}^{-1} \circ \mathcal{Q} (x)$ for all $x \in \mathbb{S}^D_H$. Hence, $\mathcal{Q}^{-1}$ is the inverse map of $\mathcal{Q}$ --- a bijection. Finally, $\mathcal{Q}$ is an isometry between $\mathbb{H}_H^D$ and $\mathbb{S}^K$ since $d(x_1,x_2) = d(\mathcal{Q}(x_1),\mathcal{Q}(x_2))$ for all $x_1, x_2 \in \mathbb{H}^D_H$.

\subsubsection{\Cref{thm:hyperbolic_pca_problem}}
 WLOG, we scale $\{h^{\prime}_{k^{\prime}}\}_{k^{\prime}}$ and $H^{\prime} \bydef [h^{\prime}_1, \ldots, h^{\prime}_{K^{\prime}}] \in \mathbb{R}^{(D+1) \times K^{\prime}}$ such that  $[h^{\prime}_{i}, h^{\prime}_{j}] = \delta_{i,j}$, i.e., ${H^{\prime}}^{\top} J_D H^{\prime} = I_{K^{\prime}}$. The cost is:
\begin{align*}
 &\mathrm{cost}(\mathbb{H}^D_H| \mathcal{X})=\sum_{k \in [K^{\prime}]}  {h^{\prime}_k}^{\top} J_D C_x J_D h^{\prime}_k \\
 &=\mathrm{Tr} \{ {H^{\prime}}^{\top} J_D C_x J_D H^{\prime} \} = \mathrm{Tr} \{ {H^{\prime}}^{\top} J_D V \Lambda V^{\top} J_D H^{\prime} \} \\
 &= \mathrm{Tr} \{ W^{\top} \Lambda W \} \stackrel{(\text{a})}{=}\mathrm{Tr} \{ W W^{\top} J_D \Lambda J_D \} = \mathrm{Tr} \{ \mathcal{W}  \Lambda J_D  \},
\end{align*}    
where $C_x = V \Lambda V^{\top}$ is $J_D$-diagonalizable, $W = V^{\top} J_D H^{\prime} \in \mathbb{R}^{(D+1) \times K^{\prime}}$, $\text{(a)}$ follows from $\Lambda$ being a diagonal matrix, i.e., $\Lambda = J_D \Lambda J_D$, and $\mathcal{W} \bydef  W W^{\top} J_D$.
\begin{lemma}\label{lem:F_trace}
$W^{\top} J_D W  = I_{K^{\prime}}$. 
\end{lemma}
\begin{proof}
$ W^{\top} J_D W  = {H^{\prime}}^{\top} J_D  (V J_D V^{\top}) J_D  H^{\prime} \stackrel{\text{(a)}}{=} {H^{\prime}}^{\top} J_D  H^{\prime}  = I_{K^{\prime}}$,
where $\text{(a)}$ follows from $V J_D V^{\top} = J_D$. This is the case, since by definition, we have $V^{\top} J_D V = J_D$, or $J_D V^{\top} J_D V_D = I_{D+1}$, that is, $J_D V^{\top}J_D= V^{-1}$. Hence, we have $V J_D V^{\top} = J_D$. Finally, ${H^{\prime}}^{\top} J_D  H^{\prime}  = I_{K^{\prime}}$ is the direct result of the orthogonality of basis vectors $h^{\prime}_1, \ldots, h^{\prime}_{K^{\prime}}$.
\end{proof}
We write the cost function as follows:
\[
    \mathrm{cost}(\mathbb{H}^D_H| \mathcal{X})= \mathrm{Tr} \{ \mathcal{W}  \Lambda J_D  \} = -\mathcal{W}_{11} \lambda_1 + \sum_{d=2}^{D+1} \mathcal{W}_{d,d} \lambda_d
\]
where $\sum_{d = 1}^{D+1} \mathcal{W}_{d,d} = \mathrm{Tr} \{ W^{\top} J_D W \} =  \mathrm{Tr} \{I_{K^{\prime}} \}$, i.e., $\mathcal{W}_{11} = -\sum_{d = 2}^{D+1} \mathcal{W}_{d,d} + K^{\prime}$. Let $W \in \mathbb{R}^{(D+1) \times K^{\prime}}$ be as follows:
\begin{equation}\label{eq:V_parameterized}
    W = \begin{bmatrix}
\sqrt{\|w_1 \|_2^2-1} & \ldots & \sqrt{\| w_{K^{\prime}}\|_2^2-1}\\
w_1 & \ldots  & w_{K^{\prime}}
\end{bmatrix},
\end{equation}
for vectors $w_1, \ldots, w_{K^{\prime}} \in \mathbb{R}^{D}$ with $\ell_2$ norms greater or equal to $1$ --- notice $W^{\top} J_D W  = I_{K^{\prime}}$ and \Cref{lem:F_trace}. From $\mathcal{W} =  W W^{\top}J_D$ and equation \eqref{eq:V_parameterized}, we have:
\[
    \mathcal{W}_{11} = -1 \sum_{k \in [K^{\prime}]} (\| w_k\|^2-1) \leq 0: \| w_k\|_2 \geq 1, \forall k \in [K^{\prime}].
\]
For $d \geq 2$, $\mathcal{W}_{d,d}$ is the squared norm of the $d$-th row of $W$. Let $W_c \in \mathbb{R}^{D \times K^{\prime}}$ where its $d$-th row is equal to the $(d+1)$-th row of $W$. Therefore, we have
\begin{align}\label{eq:W_ii}
 \sum_{d = 2}^{D+1} \mathcal{W}_{d,d} &=\mathrm{Tr} \{ W_c^{\top} W_c  \} = \sum_{k=1}^{K^{\prime}}  \| w_k \|_2^2.
\end{align}
Let us now simplify the cost function as follows:
\begin{align}
\mathrm{cost}(\mathbb{H}^D_H) &= \sum_{d = 2}^{D+1} \mathcal{W}_{d,d} (\lambda_1 + \lambda_d) - K^{\prime}D \lambda_1 \label{eq:cost_in_wii}.        
\end{align}
%(\sum_{d = 2}^{D+1} \mathcal{W}_{d,d} - K^{\prime})\lambda_1 + \sum_{d=2}^{D+1} \mathcal{W}_{d,d} \lambda_d  \nonumber \\
\begin{lemma}\label{lem:eigenvalues_are_positive}
For all $d \geq 2$, we have $\lambda_d + \lambda_1 \geq 0$.
\end{lemma}
\begin{proof}
Let $v_d$ be the $d$-th $J_D$-eigenvector of $C_x$. We have $\lambda_d = v_d^{\top} J_D \big( \mathrm{sgn}([v_d,v_d]) \lambda_d v_d \big) = v_d^{\top} J_D C_x J_D  v_d = N^{-1} \sum_{n \in [N]} [x_n, v_d]^2 \geq 0$, for all $d \geq 1$. QED.
\end{proof}

From \Cref{lem:eigenvalues_are_positive}, equation \eqref{eq:W_ii}, and $\| w_k \|_2 \geq 1$ for all $k \in [K^{\prime}]$, the minimum of the cost function in equation \eqref{eq:cost_in_wii} happens only if $\sum_{d = 2}^{D+1} \mathcal{W}_{d,d} = \sum_{k=1}^{K^{\prime}}  \|w_k\|_2^2  = K^{\prime}$, i.e., $\|w_1\|_2 = \cdots = \|w_{K^{\prime}} \|_2 = 1$. Therefore, we have $W = V^{\top} J_D H^{\prime} = \begin{bmatrix} 0 & \cdots & 0 \\ w_1 & \cdots & w_{K^{\prime}} \end{bmatrix}$. The first row of $V^{\top} J_D H^{\prime}$ corresponds to the Lorentzian product of $h^{\prime}_1, \ldots, h^{\prime}_{K^{\prime}}$ and the first column of $V$, i.e., $v_1$. Hence, we have $ h^{\prime}_1, \ldots, h^{\prime}_{K^{\prime}} \in v_{1}^{\perp}$. Since $v_1$ is the only negative $J_D$-eigenvector of $C_x$, then we have $p = v_1$. From the unit norm constraintS for $w_1, \ldots, w_{K^{\prime}}$, we have $W^{\top} W = I_{K^{\prime}}$, i.e., $W$ is a sliced unitary matrix and $W W^{\top}$ has zero-one eigenvalues. The cost function $\mathrm{cost}(\mathbb{H}^D_H| \mathcal{X})  = \mathrm{Tr} \{ W^{\top} \Lambda W \} = \mathrm{Tr} \{  \Lambda W W^{\top} \}$  achieves its minimum if and only if the non-zero singular values of $ W W^{\top}$ are aligned with the $K^{\prime}$ smallest diagonal values of $\Lambda$ --- from the von Neumann's trace inequality \cite{mirsky1975trace}. Let $\lambda_2 \leq \lambda_3 \leq \cdots$. If $w_i = e_i$ for $i = 2, \ldots, K^{\prime}+1$, then we achieve the minimum of the cost function, that is,  $h^{\prime}_1, \ldots, h^{\prime}_{K^{\prime}}$ are $K^{\prime}$ negative $J_D$-eigenvectors paired to the smallest $J_D$-eigenvalues of $C_x$.

\section{Additional Experiments}\label{sec:more_exps}
We present experiments on the \textsc{Newsgroups} dataset to demonstrate the impact of spherical PCA on classification accuracy. Using random forest and a five-layer neural network classifiers with a $90\%$ training and $10\%$ test split, the average accuracy is based on $20$ random splits. As shown in \Cref{fig:spca_experiment_classification}, SFPCA  outperforms PGA in average accuracy over $K$ (target dimension) experiments. Both methods significantly improve random forest performance and modestly improve the neural network's. This may be due to the neural network's denoising ability. For large $K$, random forest accuracy declines, unlike the neural network. This may be due to the \textsc{Newsgroups} dataset's high sparsity. Discarding small eigenvalues of the second-moment matrix significantly alters the data's sparsity, adversely affecting the random forest's accuracy.

\begin{figure}[t!]
\centering
\includegraphics[width=.5\textwidth]{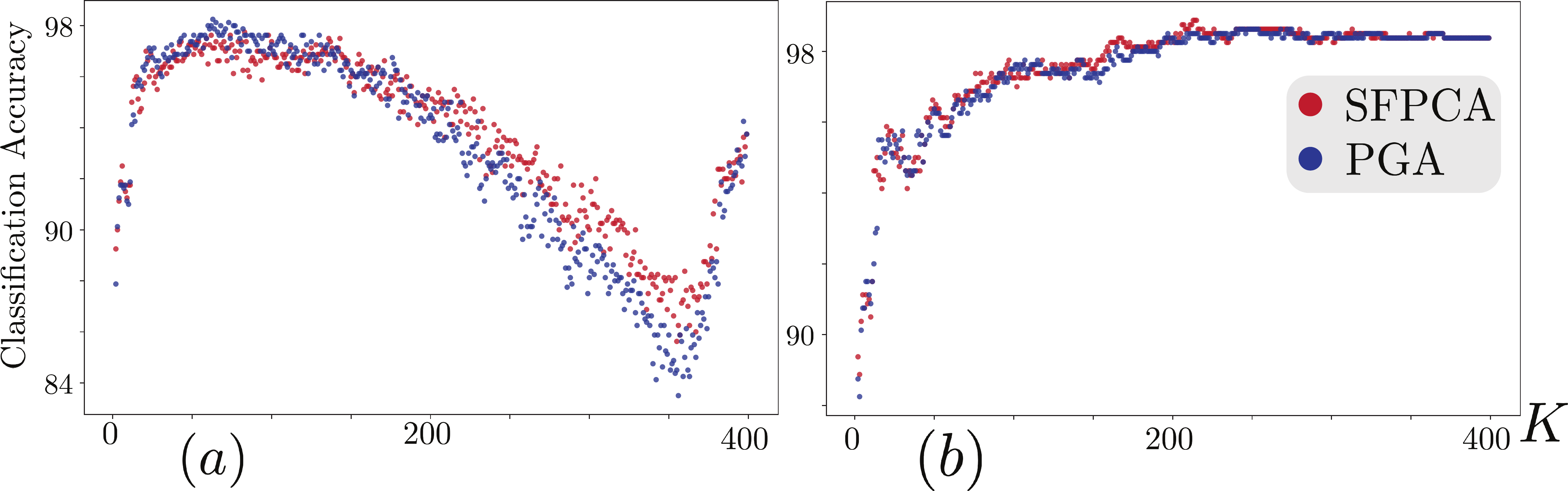}
\caption{Classification accuracy (\%) of PCA to dimension $K$ on reconstructed compositional data with $(a)$ random forest and $(b)$ neural net.}
\label{fig:spca_experiment_classification}
\end{figure}

\bibliographystyle{siamplain}
\bibliography{refs}

\vfill

\end{document}